\documentclass{article} % For LaTeX2e
\usepackage{hyperref}
\usepackage{url}
\usepackage{amsthm}
\usepackage{amsmath}
\usepackage{latexsym}
\usepackage[pdftex]{graphicx}
\usepackage{amsfonts} 
\usepackage{multirow}
\usepackage{float}
\usepackage{mathrsfs}
\usepackage[ruled]{algorithm2e}

\newtheorem{lemma}{{\bf Lemma}}

\newtheorem{proposition}{{\bf Proposition}}
\newtheorem{theorem}{{\bf Theorem}}

\def\newline{\nonumber\\&\quad}

\def\E{\mathbb{E}}

\def\Z{\mathbb{Z}}
\def\R{\mathbb{R}}
\def\IR{\mathcal{I}}
\def\a{\alpha}
\def\abs#1{{\left| #1 \right|}}
\def\I#1{{\bf 1}_{#1}}

\title{When are Kalman-Filter Restless Bandits Indexable?}
\author{Christopher Dance and Tomi Silander\\
Xerox Research Centre Europe, Grenoble, France}
\date{May, 2015}
\begin{document}
\maketitle
\begin{abstract}
We study the restless bandit associated with an extremely simple scalar
Kalman filter model in discrete time. 
Under certain assumptions, we prove that the problem is {\it indexable}
in the sense that the {\it Whittle index} is a non-decreasing function of the 
relevant belief state. In spite of the long history of this problem, this appears to be
the first such proof. 
We use results about {\it Schur-convexity} 
and {\it mechanical words},
which are particular
binary strings intimately related to {\it palindromes}.
\end{abstract}

\section{Introduction}
We study the problem of monitoring several time series so as to maintain a precise belief while minimising the cost of sensing. 
Such problems can be viewed as POMDPs with belief-dependent rewards \cite{ArayaNIPS10} and their applications include active sensing \cite{ChenICML14}, attention mechanisms for multiple-object tracking \cite{VulNIPS09}, as well as online summarisation of massive data from time-series \cite{BadanidiyuruKDD14}. Specifically, we discuss the restless bandit \cite{Whittle88} associated with the discrete-time Kalman filter~\cite{Thiele1880}. 

{\it Restless bandits} generalise bandit problems \cite{Bubeck12, Gittins11} to situations where the state of each arm (project, site or target) continues to change even if the arm is not played. As with bandit problems, the states of the arms evolve independently given the actions taken, suggesting that there might be efficient algorithms for large-scale settings, based on calculating an {\it index} for each arm, which is a real number associated with the (belief-)state of that arm alone. However, while bandits always have an optimal index policy (select the arm with the largest index), it is known that no index policy can be optimal for some discrete-state restless bandits \cite{Ortner12} and such problems are in general PSPACE-hard even to approximate to any non-trivial factor~\cite{Guha10}. Further, in this paper we address restless bandits
with real-valued rather than discrete states.

On the other hand, Whittle proposed a natural index policy for restless bandits~\cite{Whittle88}, but this policy only makes sense when the restless bandit is {\it indexable}, as we now explain.
Say we have $n$ restless bandits and we are constrained to play $m$ arms at each time.
Whittle considered relaxing this constraint by only requiring that the time-average number of arms
played is $m$.
Now the optimal average cost for this relaxed problem is a lower bound on the optimal average cost
for the original problem.
Also, the relaxed problem can be separated into $n$ single-arm problems by the method of Lagrange
multipliers, making it relatively easy to solve.
In this separated version of the relaxed problem, 
each arm behaves identically to an arm in the original problem, except that an additional
price $\lambda$ is charged each time the arm is played, where $\lambda$ corresponds to the Lagrange
multiplier for the relaxed constraint.
Now let us consider a family of optimal policies which achieves the optimal cost-to-go $Q_i(x,u;\lambda)$
for a single arm $i$ with price $\lambda$ and which takes actions $u = \pi_i(x;\lambda)$ when in state $x$
where $u = 0$ means passive and $u=1$ means active.
At first glance, we might intuitively suppose that 
it becomes less and less attractive to be active as the price $\lambda$ increases
so that as the price is increased beyond some value $\lambda_i(x)$, the optimal action
switches from active to passive.
At this price we are ambivalent between being active and passive so that 
$Q_i(x,0;\lambda_i(x)) = Q_i(x,1; \lambda_i(x))$.
Such a value $\lambda_i(x)$ is called the {\it Whittle index} for arm $i$ in state $x$.
Indeed if there is a family of optimal policies for which 
\begin{align*}
\pi_i(x;\lambda_{\text{hi}}) \le \pi_i(x;\lambda_{\text{lo}}) \quad \text{for all states $x$ and all pairs of prices $\lambda_{\text{hi}} \ge \lambda_{\text{lo}}$}
\end{align*}
then an optimal solution to the relaxed problem for price $\lambda$
is to activate arm $i$ if and only if $\lambda < \lambda_i(x)$.
If a restless bandit satisfies this condition, it is said to be {\it indexable}.
It is important to note that some restless bandits are not indexable,
so activating arm $i$ if and only if $\lambda < \lambda_i(x)$ does not correspond to an optimal solution 
to the relaxed problem.
Indeed, in a study of small randomly-generated problems, 
Weber and Weiss~\cite{Weber90} found that roughly 10\% of problems were not indexable.

As a policy based on $\lambda_i(x)$ is so good for the relaxed problem when the arms are indexable,
this motivates us to use $\lambda_i(x)$ as a heuristic for the original problem.
This heuristic is called Whittle's index policy and at each time it
activates the $m$ arms with the highest indexes $\lambda_i(x)$.
Further motivation for studying indexability is that for ordinary bandits 
the Whittle index reduces to the Gittins index,
making the Whittle index policy optimal when only one arm may be active at each time,
that is when $m=1$.
More generally, Whittle's index policy is not optimal for some restless bandit problems even when the arms
are indexable, but indexability is still a rather useful concept, 
since if all arms are indexable and certain other conditions hold, Whittle's policy is asymptotically
optimal, as we now explain. 
Consider a sequence of restless bandit problems parameterised by the number of
indexable arms $n$ and in which $m = \alpha n$ of the arms can be simultaneously active 
for some fixed $\alpha \in (0,1)$.
Then as $n$ tends to infinity, 
the time-average cost per arm for Whittle's index policy converges to the time-average
cost per arm for an optimal policy, provided a certain fluid approximation has a unique fixed point.
This result was first demonstrated by Weber and Weiss~\cite{Weber90} who for simplicity of exposition
only considered the symmetric case in which the $n$ arms have identical costs and transition probabilities.
Recently, Verloop~\cite{Verloop14} extended this result to asymmetric cases 
involving multiple types of arms. 
Interestingly, this extension also covers cases where new arms arrive and old arms depart.

Restless bandits associated with scalar Kalman(-Bucy) filters in continuous time were recently shown to be indexable~\cite{LeNy11} and the corresponding discrete-time problem has attracted considerable attention over a long period~\cite{Meier67, LaScala06, NinoMora09, Villar12}.
However, that attention has produced no satisfactory proof of indexability -- {\it even} for scalar time-series and even if we assume that there is 
a {\it monotone} optimal policy for the single-arm problem, 
which is a policy that plays the arm if and only if the relevant belief-state 
exceeds some threshold (here the relevant belief-state is a posterior variance).
Theorem~\ref{main} of this paper addresses that gap. After formalising the problem (Section 2), we describe the concepts and intuition (Section~3) behind the main result (Section 4). The main tools are {\it mechanical words} (which are not sufficiently well-known) and {\it Schur convexity}. As these tools are associated with rather general theorems, we believe that future work (Section 5) should enable substantial generalisation of our results.

\section{Problem and Index}
We consider the problem of tracking $N$ time-series, which we call arms,
in discrete time. 
The state $Z_{i,t}\in \R$ of arm $i$ at time $t\in \Z_+$ evolves
as a standard-normal random walk independent 
of everything but its immediate past ($\Z_+, \R_-$ and $\R_+$ all include zero).
The action space is $\mathcal{U} := \{1, \dots, N\}$.
Action $u_t=i$ makes an expensive observation $Y_{i,t}$ of arm $i$
which is normally-distributed about $Z_{i,t}$ 
with precision $b_i\in\R_+$ and we receive 
cheap observations $Y_{j,t}$ of 
each other arm $j$ with precision $a_j\in\R_+$
where $a_j<b_j$ and $a_j=0$ means no observation at all.

Let $Z_t, Y_t, \mathcal{H}_t, \mathcal{F}_t$ be the state, observation,
history and observed history, so that
$Z_t:=(Z_{1,t}, \dots, Z_{N,t}), 
Y_t:= (Y_{1,t}, \dots, Y_{N,t}),
\mathcal{H}_t:= ((Z_{0},u_0,Y_{0}), \dots, (Z_{t},u_t,Y_{t}))$ and 
$\mathcal{F}_t := ((u_{0}, Y_{0}), \dots, (u_{t},Y_{t})).$
Then we formalise the above as ($\I{\cdot}$ is the indicator function)
\begin{align*}
Z_{i,0} &\sim \mathcal{N}(0,1), 
& Z_{i,t+1}\mid\mathcal{H}_{t} &\sim \mathcal{N}(Z_{i,t}, 1), 
& Y_{i,t}\mid\mathcal{H}_{t-1}, Z_{t}, u_t &\sim \mathcal{N}\left( Z_{i,t}, 
\frac{\I{u_t\neq  i}}{a_i}+\frac{\I{u_t=i}}{b_i}\right) .
\end{align*}
Note that this setting is readily generalised to 
$\E[(Z_{i,t+1}-Z_{i,t})^2]\neq 1$ by a change of variables.

Thus the posterior belief is given by the Kalman filter as
$Z_{i,t}\mid\mathcal{F}_t \sim \mathcal{N}(\hat Z_{i,t}, x_{i,t})$
where the posterior mean is ${\hat Z}_{i,t}\in \R$ and the {\it error variance} $x_{i,t}\in\R_+$ satisfies
\begin{align}
x_{i,t+1} = \phi_{i,\I{u_{t+1}=i}}(x_{i,t}) \quad \text{where} \quad
\phi_{i,0}(x) := \frac{x+1}{a_ix+a_i+1} \ \ \text{and} \ \
\phi_{i,1}(x) := \frac{x+1}{b_ix+b_i+1} .
\label{phi}
\end{align}

{\bf Problem KF1.} 
Let $\pi$ be a policy so that $u_t = \pi(\mathcal{F}_{t-1})$.
Let $x^\pi_{i,t}$ be the error variance under $\pi$.
The problem is to choose $\pi$ so as to minimise the following 
objective for discount factor $\beta\in [0,1)$. 
The objective consists of a weighted sum of error variances $x^\pi_{i,t}$
with weights $w_i\in\R_+$ plus observation costs 
$h_i\in\R_+$ for $i=1, \dots, N$:
\begin{align*}
\E \left[ 
	\sum_{t=0}^\infty \sum_{i=1}^N \beta^t \left\{ 
	 h_i \I{u_t=i} + w_i x_{i,t}^\pi \right\}
\right] = \sum_{t=0}^\infty \sum_{i=1}^N \beta^t \left\{ 
	 h_i \I{u_t=i} + w_i x_{i,t}^\pi \right\}
\end{align*}
where the equality follows as~(\ref{phi}) is a deterministic mapping 
(and assuming $\pi$ is deterministic). 

{\bf Single-Arm Problem and Whittle Index.} Now fix an arm $i$ and write
$x_t^\pi,\phi_0(\cdot), \dots$  instead of $x_{t,i}^\pi, \phi_{i,0}(\cdot), \dots$.
Say there are now two actions $u_t=0,1$ corresponding to cheap and expensive observations respectively
and the expensive observation now costs $h+\nu$ where $\nu\in\R$.
The {\it single-arm problem} is to 
choose a policy, which here is an action sequence, $\pi:=(u_0, u_1, \dots)$ 
\begin{align}
\text{so as to minimise} \quad V^\pi(x|\nu) := 
\sum_{t=0}^\infty \beta^t \left\{ 
	 (h+\nu) u_t + w x_{t}^\pi \right\} \quad \text{where $x_0 = x$.}
\label{subproblem}
\end{align}
Let $Q(x,\alpha|\nu)$ be the optimal cost-to-go in this problem if the first
action must be $\alpha$ and let $\pi^*$ be an optimal policy, so that
\begin{align*}
Q(x,\alpha|\nu) := (h+\nu) \alpha + w x + \beta V^{\pi^*}(\phi_\alpha(x)|\nu) .
\end{align*}
For any fixed $x\in\R_+$, the value of $\nu$ for which actions $u_0 = 0$
and $u_0 = 1$ are both optimal is known as the {\it Whittle index} $\lambda^W(x)$
assuming it exists and is unique. 
In other words 
\begin{align}
\text{\it The Whittle index $\lambda^W(x)$ is the solution to $Q(x,0|\lambda^W(x))=Q(x,1|\lambda^W(x)).$}
\label{Qsinglearm}
\end{align}
Let us consider a policy which takes action $u_0=\alpha$ then acts optimally producing actions $u_t^{\alpha*}(x)$ and error variances $x_t^{\alpha *}(x)$.
Then~(\ref{Qsinglearm}) gives
\begin{align*}
\sum_{t=0}^\infty \beta^t \left\{ 
	 (h+\lambda^W(x)) u^{0*}_t + w x_{t}^{0*}(x) \right\}
= \sum_{t=0}^\infty \beta^t \left\{ 
	 (h+\lambda^W(x)) u^{1*}_t + w x_{t}^{1*}(x) \right\}.
\end{align*}
Solving this linear equation for the index $\lambda^W(x)$ gives
\begin{align}
\lambda^W(x) = w \frac{\sum_{t=1}^\infty \beta^t (x_{t}^{0*}(x)-x_{t}^{1*}(x))}{
\sum_{t=0}^\infty \beta^t (u^{1*}_t(x)-u^{0*}_t(x))} - h.
\label{index1}
\end{align}
Whittle~\cite{Whittle88} recognised that for his index policy (play the arm with the largest $\lambda^W(x)$) to make sense, 
any arm which receives an expensive observation for 
added cost $\nu$, must also receive an expensive observation for 
added cost $\nu'<\nu$.
Such problems are said to be {\it indexable}.
The question resolved by this paper is whether 
Problem KF1 is indexable.
Equivalently, is $\lambda^W(x)$ non-decreasing in $x\in\R_+$?

\section{Main Result, Key Concepts and Intuition}
We make the following intuitive assumption about threshold (monotone) policies. 
\vspace{0.2cm} \\
{\bf A1.} {\it For some $x\in\R_+$ depending on $\nu\in\R$, the policy $u_t = \I{x_t \ge x}$ is optimal for problem~(\ref{subproblem}).} \vspace{0.2cm}

Note that under A1, definition~(\ref{Qsinglearm}) means
the policy $u_t = \I{x_t > x}$ is also optimal, so we can choose
\begin{align}
\left.\begin{aligned}
u_{t}^{0*}(x) &:= \begin{cases} 0 & \text{if $x_{t-1}^{0*}(x)\le x$} \\ 1 & \text{otherwise} \end{cases}
& \quad\text{and}\quad
x_{t}^{0*}(x) &:= \begin{cases} \phi_0(x_{t-1}^{0*}(x)) & \text{if $x_{t-1}^{0*}(x)\le x$} \\ \phi_1(x_{t-1}^{0*}(x))  & \text{otherwise} \end{cases}
\\
u_{t}^{1*}(x) &:= \begin{cases} 0 & \text{if $x_{t-1}^{1*}(x)< x$} \\ 1 & \text{otherwise} \end{cases} 
& \quad\text{and}\quad
x_{t}^{1*}(x) &:= \begin{cases} \phi_0(x_{t-1}^{1*}(x)) & \text{if $x_{t-1}^{1*}(x)< x$} \\ \phi_1(x_{t-1}^{1*}(x))  & \text{otherwise} \end{cases}
\end{aligned}
\quad\right\}
\label{orbits}
\end{align}
where $x_0^{0*}(x)=x_0^{1*}(x)=x$.
We refer to $x_t^{0*}(x), x_t^{1*}(x)$ as the {\it $x$-threshold orbits}
(Figure~\ref{fig:orbits}).

We are now ready to state our main result.

{\bf Theorem 1.} {\it
Suppose a threshold policy (A1) is optimal for the single-arm problem~(\ref{subproblem}). 
Then Problem KF1 is indexable. 
Specifically, for any $b>a\ge 0$ let
\begin{align*}
\phi_0(x)&:=\frac{x+1}{a x+a+1}, &\phi_1(x)&:=\frac{x+1}{bx+b+1}
\end{align*}
and for any $w\in\R_+, h\in\R$ and $0<\beta<1$, let
\begin{align}
\lambda^W(x) := w \frac{\sum_{t=1}^\infty \beta^t (x_{t}^{0*}(x)-x_{t}^{1*}(x))}{
\sum_{t=0}^\infty \beta^t (u^{1*}_t(x)-u^{0*}_t(x))} - h
\label{index}
\end{align}
in which action sequences $u_{t}^{0*}(x), u_{t}^{1*}(x)$
and error variance sequences $x_{t}^{0*}(x), x_{t}^{1*}(x)$
are given in terms of $\phi_0, \phi_1$ by~(\ref{orbits}).
Then $\lambda^W(x)$ is a continuous and non-decreasing function of $x\in \R_+$.}

\begin{figure}
\centering
\includegraphics[bb=15mm 105mm 195mm 200mm,clip=true,width=12.8cm]{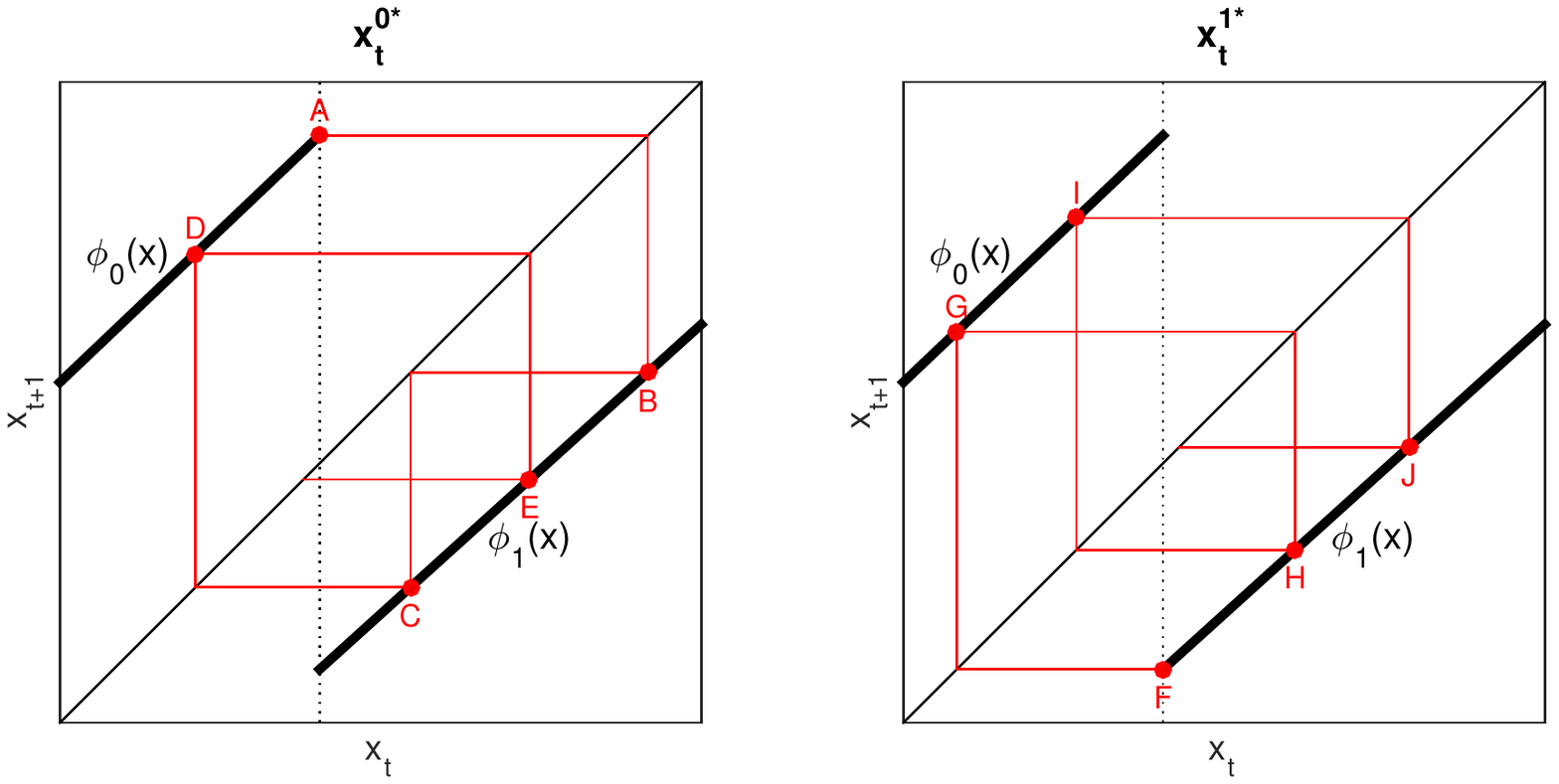}
\caption{Orbit $x^{0*}_t(x)$ traces the path $ABCDE\dots$ for the word $01w=01101$.
Orbit $x^{1*}_t(x)$ traces the path $FGHIJ\dots$ 
for the word $10w=10101$.
Word $w=101$ is a palindrome.}
\label{fig:orbits}
\end{figure}

We are now ready to describe the key concepts underlying this result.

{\bf Words.} 
In this paper, a {\it word} $w$ is a string on $\{0,1\}^*$ with $k^{\rm th}$ letter $w_k$ and $w_{i:j}:=w_{i} w_{i+1} \dots w_{j}$. The empty word is $\epsilon$, the concatenation of words $u,v$ is $uv$,
the word that is the $n$-fold repetition of $w$ is $w^n$,
the infinite repetition of $w$ is $w^\omega$ 
and $\tilde w$ is the reverse of $w$,
so $w=\tilde w$ means $w$ is a palindrome.
The length of $w$ is $\abs{w}$ and $\abs{w}_u$ is the number
of times that word $u$ appears in $w$, overlaps included.

{\bf Christoffel, Sturmian and Mechanical Words.}
It turns out that the action sequences in~(\ref{orbits}) are given
by such words, so the following definitions are central to this paper.

The {\it Christoffel tree} (Figure~\ref{fig:tree}) is an infinite complete binary tree~\cite{Berstel08} in which each node is labelled with a pair $(u,v)$ of words. 
The root is $(0,1)$ and the children of $(u,v)$ are $(u,uv)$ and $(uv,v)$.
The {\it Christoffel words} are the words $0,1$ and the concatenations $uv$ for all $(u,v)$ in that tree.
The fractions $\abs{uv}_1 / \abs{uv}_0$ form the Stern-Brocot tree~\cite{Graham94} which contains each positive rational number exactly once.
Also, infinite paths in the Stern-Brocot tree converge to the positive irrational numbers. Analogously, {\it Sturmian words} could be thought of 
as infinitely-long Christoffel words.

Alternatively, among many known characterisations, the Christoffel words can be defined as the words $0,1$
and the words $0w1$ where $a:=\abs{0w1}_1 / \abs{0w1}$ and 
\begin{align*}
(01w)_n := \lfloor (n+1)a \rfloor - \lfloor na \rfloor
\end{align*}
for any relatively prime natural numbers $\abs{0w1}_0$ and $\abs{0w1}_1$
and for $n=1, 2, \dots, \abs{0w1}$.
The Sturmian words are then the infinite words $0w_1w_2\cdots$ where,
for $n=1, 2, \dots$ and $a\in (0,1) \backslash \mathbb{Q}$,
\begin{align*}
(01w_1w_2\cdots)_n := \lfloor (n+1)a \rfloor - \lfloor na \rfloor .
\end{align*}
We use the notation $0w1$ for Sturmian words although they are infinite.
%Although Sturmian words are infinite, we still write them as $0w1$.
\begin{figure}
\centering
\includegraphics[bb=20mm 130mm 195mm 165mm,clip=true,width=12.8cm]{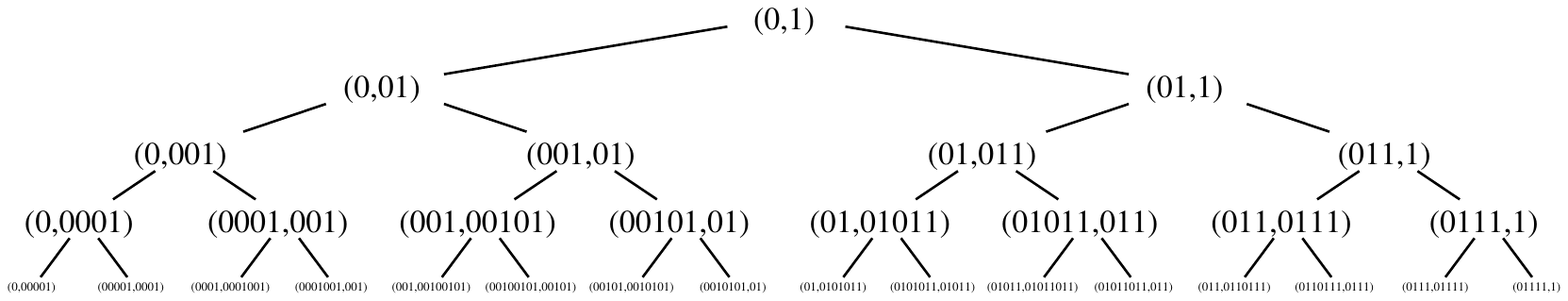}
\caption{Part of the Christoffel tree.}
\label{fig:tree}
\end{figure}

The set of {\it mechanical words} is the union of the Christoffel and Sturmian  words~\cite{Lothaire02}. (Note that the mechanical words are sometimes
defined in terms of infinite repetitions of the Christoffel words.)

{\bf Majorisation.}
As in~\cite{Marshall10}, let $x,y\in \R^m$ and let $x_{(i)}$ and $y_{(i)}$ be their elements sorted in {\it ascending} order.
We say $x$ is {\it weakly supermajorised} by $y$ and write
$x\prec^w y$ if 
\begin{align*}
\sum_{k=1}^j x_{(k)} \ge \sum_{k=1}^j y_{(k)}\qquad\text{for all $j=1, \dots, m$.}
\end{align*}
If this is an equality for $j=m$ we say $x$ is {\it majorised} by $y$
and write $x\prec y$.
It turns out that 
\begin{align*}
x\prec y \qquad &\Leftrightarrow \qquad \sum_{k=1}^j x_{[k]} \le \sum_{k=1}^j y_{[k]} \quad \text{for $j=1, \dots, m-1$ with equality for $j=m$}
\intertext{
where $x_{[k]},y_{[k]}$ are the sequences sorted in {\it descending} order.
For $x,y\in\R^m$ we have~\cite{Marshall10}}
x\prec y \qquad &\Leftrightarrow\qquad
\sum_{i=1}^m f(x_i) \le \sum_{i=1}^m f(y_i) \quad\text{for all convex functions
$f:\R\rightarrow \R$.}
\end{align*}
More generally, a real-valued function $\phi$ defined on a subset $\mathcal{A}$ of $\R^m$ is said to be {\it Schur-convex}
on $\mathcal{A}$ if $x\prec y$ implies that $\phi(x)\le\phi(y)$.

{\bf M{\"o}bius Transformations.} Let $\mu_A(x)$ denote the M{\"o}bius transformation $\mu_A(x) := \frac{A_{11}x+A_{12}}{A_{21}x+A_{22}}$ where $A\in\R^{2\times 2}$.
M{\"o}bius transformations such as 
$\phi_0(\cdot), \phi_1(\cdot)$ are closed under composition, 
so 
for any word $w$ we define 
$\phi_w(x):=\phi_{w_{\abs{w}}}\circ \dots \circ \phi_{w_2} \circ \phi_{w_1}(x)$ and $\phi_\epsilon(x) := x.$

{\bf Intuition.} Here is the intuition behind our main result.
 
For any $x\in \R_+$, the orbits in~(\ref{orbits}) 
correspond to a particular mechanical word $0,1$ 
or $0w1$ depending on the value of $x$
(Figure~\ref{fig:orbits}).
Specifically, for any word $u$, let $y_u$ be the fixed point of the mapping
$\phi_u$ on $\R_+$ so that $\phi_u(y_u)=y_u$ and $y_u\in\R_+$.
Then the word corresponding to $x$ is 1 for
$0\le x\le y_1$, $0w1$ for $x\in [y_{01w},y_{10w}]$ and 0 for 
$y_0\le x < \infty$.
In passing we note that these fixed points 
are sorted in ascending order
by the ratio $\rho:=\abs{01w}_0 / \abs{01w}_1$ of counts of 0s to counts of 1s,
as illustrated by Figure~\ref{fig:intuition}.
Interestingly, it turns out that ratio $\rho$ 
is a piecewise-constant yet continuous function of $x$,
reminiscent of the Cantor function.

Also, composition of M{\"o}bius transformations 
is homeomorphic to matrix multiplication
so that 
\begin{align*}
\mu_A \circ \mu_B(x) = \mu_{AB}(x) \qquad\text{for any $A,B\in\R^{2\times 2}.$} 
\end{align*}
Thus, the index~(\ref{index}) can be written in terms of the orbits of a linear system (\ref{sigmas}) given by $0, 1$ or $0w1.$
Further, if $A\in\R^{2\times2}$ and $\det(A)=1$ 
then the gradient of the corresponding M{\"o}bius transformation is the convex function
\begin{align*}
\frac{d\mu_A(x)}{dx} = \frac{1}{(A_{21}x+A_{22})^{2}}.
\end{align*}
So the gradient of the index
is the difference of the sums of a convex function of the 
linear-system orbits.
However, such sums are Schur-convex functions
and it follows that the index is increasing because one orbit
weakly supermajorises the other, as we now show for the case
$0w1$ (noting that the proof is easier for words $0,1$).
As $0w1$ is a mechanical word, $w$ is a palindrome.
Further, if $w$ is a palindrome, it turns out that 
the difference between the linear-system orbits increases with $x$.
So, we might define the {\it majorisation point} for $w$ as the $x$ for which 
one orbit majorises the other.
Quite remarkably, 
if $w$ is a palindrome then the majorisation point is $\phi_w(0)$ (Proposition~\ref{major}).
Indeed the black circles and blue dots of Figure~\ref{fig:intuition} coincide.
Finally, $\phi_w(0)$ is less than or equal to $y_{01w}$
which is the least $x$
for which the orbits correspond to the word $0w1$. 
Indeed, the blue dots of Figure~\ref{fig:intuition} are below the corresponding black dots. 
Thus one orbit does indeed supermajorise the other.

\begin{figure}
\centering
\includegraphics[bb=20mm 105mm 185mm 187mm,clip=true,width=12.8cm]{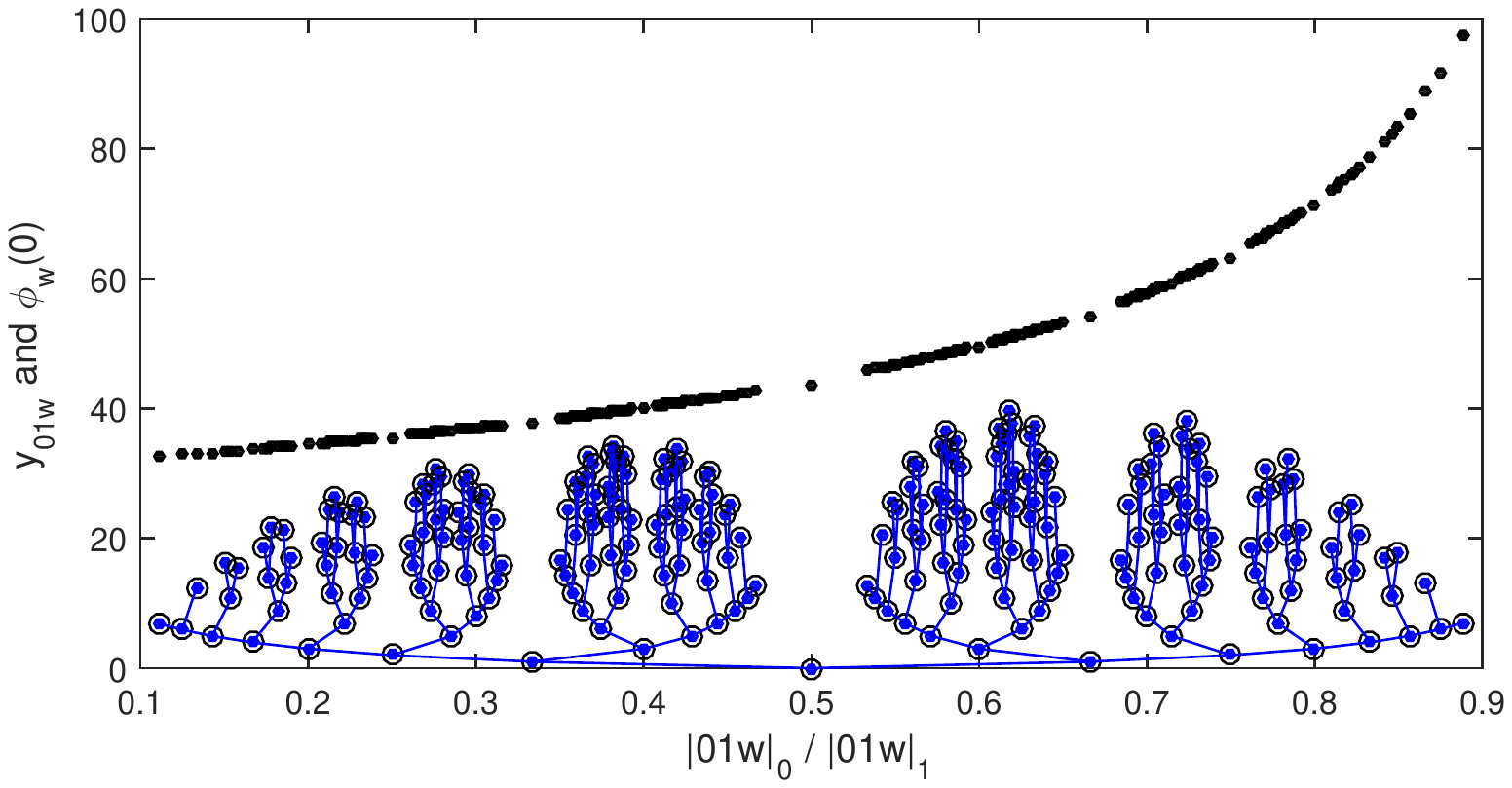}
\caption{Lower fixed points $y_{01w}$
of Christoffel words (black dots), majorisation points for those words (black circles) and the tree of $\phi_w(0)$ (blue).}
\label{fig:intuition}
\end{figure}

\section{Proof of Main Result}
\subsection{Mechanical Words}
The M{\"o}bius transformations of~(\ref{phi}) satisfy the following assumption for $\IR:= \R_+$.
We prove that the fixed point $y_w$ of word $w$
(the solution to $\phi_w(x)=x$ on $\IR$) is unique in the supplementary material.

{\bf Assumption A2.} {\it Functions $\phi_0 : \IR \rightarrow \IR, \phi_1 : \IR \rightarrow \IR$, where $\IR$ is an interval of $\mathbb{R}$, are increasing and non-expansive, so for all $x, y\in\IR : x < y$ and for $k \in \{0,1\}$ we have
\begin{align*}
\underbrace{\phi_k(x) < \phi_k(y)}_{\text{increasing}} \qquad \qquad \text{and} \qquad \qquad \underbrace{\phi_k(y) - \phi_k(x) < y-x}_{\text{non-expansive}} .
\end{align*}
Furthermore, the fixed points $y_0, y_1$ of $\phi_0, \phi_1$ on $\IR$ satisfy $y_1 < y_0$. }

Hence the following two propositions (supplementary material) apply to $\phi_0, \phi_1$ of~(\ref{phi}) on $\IR=\R_+$.
\begin{proposition} Suppose A2 holds, $x\in \IR$ and $w$ is a non-empty word. Then 
\begin{align*}
x < \phi_w(x) \ \Leftrightarrow \  \phi_w(x) < y_w \ \Leftrightarrow \  x < y_w
&& \text{and} &&
x > \phi_w(x) \ \Leftrightarrow \  \phi_w(x) > y_w \ \Leftrightarrow \  x > y_w.
\end{align*}
\label{increasing}
\end{proposition}
\vspace{-0.7cm}
For a given $x$, in the notation of~(\ref{orbits}), 
we call the shortest word $u$ such that 
$(u^{1*}_1,u^{1*}_2, \dots)=u^\omega$ the {\it $x$-threshold word}.
Proposition~\ref{mechanical} generalises a recent 
result about $x$-threshold words in 
a setting where $\phi_0, \phi_1$ are linear~\cite{Rajpathak12}.
\begin{proposition}
Suppose A2 holds and $0w1$ is a mechanical word. 
Then 
\begin{align*}
\text{$0w1$ is the $x$-threshold word} \ \Leftrightarrow \ x \in [y_{01w}, y_{10w}].
\end{align*}
Also, if $x_0, x_1\in\mathcal{I}$ with $x_0\ge y_0$ and $x_1\le y_1$ 
then the $x_0$- and $x_1$-threshold words are $0$ and $1$.
\label{mechanical}
\end{proposition}

We also use the following very interesting fact (Proposition 4.2 on p.28 of~\cite{Berstel08}). 
\begin{proposition}
Suppose $0w1$ is a mechanical word. Then $w$ is a palindrome.
\label{palindrome}
\end{proposition}
%\vfill\pagebreak

\subsection{Properties of the Linear-System Orbits $M(w)$ and Prefix Sums $S(w)$}
{\bf Definition.} Assume that $a, b\in \R_+$ and $a<b$. 
Consider the matrices 
\begin{align*}
F:=\begin{pmatrix}1&1\\a&1+a\end{pmatrix},\qquad
G:=\begin{pmatrix}1&1\\b&1+b\end{pmatrix}
\quad\text{and}\quad K:=\begin{pmatrix}-1&-1\\0&1\end{pmatrix} 
\end{align*}
so that the M{\"o}bius transformations $\mu_F,\mu_G$ are the functions $\phi_0,\phi_1$ of~(\ref{phi}) and $GF-FG=(b-a)K$.
Given any word $w\in\{0,1\}^*$, we define the {\it matrix product} $M(w)$ 
\begin{align*}
M(w) := M(w_{\abs{w}}) \cdots M(w_1), \quad \text{where $M(\epsilon):=I, M(0):=F$ and $M(1):=G$}
\end{align*}
where $I\in\R^{2\times 2}$ is the identity and the {\it prefix sum} $S(w)$ as the matrix polynomial
\begin{align}
S(w) := \sum_{k=1}^{\abs{w}} M(w_{1:k}), \qquad \text{where $S(\epsilon) := 0$ (the all-zero matrix).}
\end{align}
For any $A\in\R^{2\times 2}$, let $\text{tr}(A)$ be the trace of $A$,
let $A_{ij}=[A]_{ij}$ be the entries of $A$ and let $A\ge 0$
indicate that all entries of $A$ are non-negative.

{\bf Remark.} Clearly, $\det(F)=\det(G)=1$ so that $\det(M(w))=1$ for any word $w$. Also, $S(w)$ corresponds to the partial sums of the linear-system orbits, as hinted in the previous section. 

The following proposition captures the role of palindromes (proof in the supplementary material).
\begin{proposition}
Suppose $w$ is a word, $p$ is a palindrome and $n\in \Z_+$.
Then
\begin{enumerate}
\item $M(p)=\begin{pmatrix}\frac{fh+1}{h+f}&f\\\frac{h^2-1}{h+f}&h\end{pmatrix}$ for some $f,h\in\R$,
\item $\text{tr}(M(10p))=\text{tr}(M(01p))$,
\item If $u\in\{p(10p)^n,(10p)^n10\}$ then 
$M(u)-M(\tilde u) = \lambda K$ for some $\lambda\in\R_-$,
\item If $w$ is a prefix of $p$ then $[M(p(10p)^n10w)]_{22}\le [M(p(01p)^n01w)]_{22}$,
\item $[M((10p)^n10w)]_{21}\ge [M((01p)^n01w)]_{21}$,
\item $[M((10p)^n1)]_{21}\ge[M((01p)^n0)]_{21}$.
\end{enumerate}
\label{prop:pal}
\end{proposition}

We now demonstrate a surprisingly simple relation between
$S(w)$ and $M(w)$.

%\vfill\pagebreak

\begin{proposition}
Suppose $w$ is a palindrome.
Then 
\begin{align}
S_{21}(w)=M_{22}(w)-1\qquad\text{and}\qquad S_{22}(w)=M_{12}(w)+S_{21}(w).
\label{Sform}
\end{align}
Furthermore, if $\Delta_k := [S(10w)M(w(10w)^k)-S(01w)M(w(01w)^k)]_{22}$ then
\begin{align}
\Delta_k = 0 \qquad\text{for all $k\in \Z_+$.}
\label{Ssum}
\end{align}
\end{proposition}
\begin{proof}
Let us write $M:=M(w),S:=S(w)$.
We prove~(\ref{Sform}) by induction on $\abs{w}$.
In the base case $w\in\{\epsilon, 0, 1\}$.
For $w=\epsilon$, 
$M_{22}-1=0=S_{21}, M_{12}+S_{21}=0=S_{22}.$
For $w\in\{0,1\}$, 
$M_{22}-1=c=S_{21}, M_{12}+S_{21}=1+c=S_{22}$ for some $c\in\{a,b\}$.
For the inductive step, in accordance with Claim~1 of Proposition~\ref{prop:pal}, assume $w\in \{0v0,1v1\}$ for some word $v$ satisfying
\begin{align*}
M(v)&=\begin{pmatrix}\frac{fh+1}{h+f}&f\\\frac{h^2-1}{h+f}&h\end{pmatrix}, & 
S(v)&=\begin{pmatrix} c& d\\ h-1& f+h-1\end{pmatrix}
 \quad \text{for some $c,d,f,h\in\R$.}
\end{align*}
For $w=1v1$, $M:=M(1v1)=GM(v)G$ and $S:=S(1v1)=GM(v)G+S(v)G+G$.
Calculating the corresponding matrix products and sums gives
\begin{align*}
%S_{21}&=\frac{(b h+h+b f-1) (b h+2 h+b f+f+1)}{h+f} = M_{22}-1 \\
S_{21}&=(b h+h+b f-1) (b h+2 h+b f+f+1)(h+f)^{-1} = M_{22}-1 \\
S_{22}-S_{21}&=b h+2 h+b f+f = M_{12} 
\end{align*}
as claimed.
For $w=0u0$ the claim also holds as $F=\left. G\right|_{b=a}$.
This completes the proof of~(\ref{Sform}).

{\it Furthermore Part.}
Let $A:=S(w)FG+FG+G$ and $B:=S(w)GF+GF+F$. Then
\begin{align}
\Delta_k = [(A(M(w)FG)^k - B(M(w)GF)^k) M(w)]_{22} 
\label{Delta}
\end{align}
by definition of $S(\cdot)$. 
By Claim~1 of Proposition~\ref{prop:pal} and~(\ref{Sform}) we know that
\begin{align*}
M(w) &= \begin{pmatrix}\frac{fh+1}{h+f}&f\\\frac{h^2-1}{h+f}&h\end{pmatrix}, &
S(w)&=\begin{pmatrix} c& d\\ h-1& f+h-1\end{pmatrix}
 \quad \text{for some $c,d,f,h\in\R$.}
\end{align*}
Substituting these expressions and the definitions of $F,G$ into the definitions of $A,B$ and then into~(\ref{Delta}) for $k\in\{0,1\}$ directly gives $\Delta_0 = \Delta_1 = 0$ (although this calculation is long).

Now consider the case $k\ge 2$.
Claim~2 of Proposition~\ref{prop:pal} says $\text{tr}(M(10w))=\text{tr}(M(01w))$ and clearly $\det(M(10w))=\det(M(01w))=1$.
Thus we can diagonalise as
%$M(w)FG =: UDU^{-1}, M(w)GF =: VDV^{-1}, D:=\text{diag}(\lambda,1/\lambda)$
\begin{align*}
M(w)FG &=: UDU^{-1}, & M(w)GF &=: VDV^{-1}, & D& :=\text{diag}(\lambda,1/\lambda)
\quad\text{for some $\lambda\ge 1$}
\end{align*}
%for some eigenvalue $\lambda\ge 1$,
%\begin{align*}
%M(w)FG &=: UDU^{-1}, & M(w)GF =:& VDV^{-1}, & 
%D := \begin{pmatrix} \lambda&0\\0&1/\lambda\end{pmatrix} 
%\end{align*}
so that
%\begin{align*}
%\Delta_k &= [A UD^kU^{-1} M(w) - e^TB VD^k V^{-1} M(w)]_{22} 
%=: \gamma_1 \lambda^k + \gamma_2 \lambda^{-k} .
%\end{align*}
$\Delta_k = [A UD^kU^{-1} M(w) - e^TB VD^k V^{-1} M(w)]_{22} 
=: \gamma_1 \lambda^k + \gamma_2 \lambda^{-k} .$
So, if $\lambda=1$ then $\Delta_k = \gamma_1+\gamma_2 = \Delta_0$
and we already showed that $\Delta_0 = 0$.
Otherwise $\lambda\neq 1$, so $\Delta_0 = \Delta_1 = 0$ implies 
$\gamma_1+\gamma_2=\gamma_1\lambda+\gamma_2\lambda^{-1}=0$
which gives $\gamma_1=\gamma_2=0$.
Thus for any $k\in\Z_+$ we have
$\Delta_k = \gamma_1\lambda^k+\gamma_2\lambda^{-k}=0$.
\end{proof}

\subsection{Majorisation}
The following is a straightforward consequence of 
results in~\cite{Marshall10} proved in the supplementary material.
We emphasize that the notation $\prec^w$ has nothing to do with the notion of $w$ as a word. 
\begin{proposition}
Suppose $x,y \in \R_+^m$ and
$f : \R \rightarrow \R$ is a symmetric function that is convex and
decreasing on $\R_+$.
Then 
$\text{$x\prec^w y$ and $\beta\in [0,1]$} \quad\Rightarrow \quad\sum_{i=1}^m \beta^{i} f(x_{(i)}) \ge \sum_{i=1}^m \beta^{i} f(y_{(i)})$.
\label{symconvex}
\end{proposition}

%For any $x\in\R$ and any word $w$, define the sequences for $k\in\Z_{++}$
%\begin{align}
%x_k(x) &:= e^T M(((10w)^\omega)_{1:k}) \begin{pmatrix}x\\1\end{pmatrix},
%&
%y_k(x) &:= e^T M(((01w)^\omega)_{1:k}) \begin{pmatrix}x\\1\end{pmatrix}.
%\end{align}

For any $x\in\R$ and any fixed word $w$, define the sequences for $n\in\Z_+$ and $k=1, \dots, m$ 
\begin{align}
\left.\begin{aligned}
x_{nm+k}(x) &:= [M((10w)^n(10w)_{1:k}) v(x)]_2, &
\sigma_x^{(n)}:=(x_{nm+1}(x), \dots, x_{nm+m}(x))\\
y_{nm+k}(x) &:= [M((01w)^n(01w)_{1:k}) v(x)]_2, &
\sigma_y^{(n)}:=(y_{nm+1}(x), \dots, y_{nm+m}(x))
\end{aligned}\right\}
\label{sigmas}
\end{align}
where $m:=\abs{10w}$ and $v(x):=(x,1)^T.$

\begin{proposition}
Suppose $w$ is a palindrome and $x\ge \phi_w(0)$. 
Then $\sigma_x^{(n)}$ and $\sigma_y^{(n)}$
are ascending sequences on $\R_+$ and $\sigma_x^{(n)} \prec^w \sigma_y^{(n)}$
for any $n\in\Z_+$. 
\label{major}
\end{proposition}

\begin{proof}
Clearly $\phi_w(0) \ge 0$ so $x\ge 0$ and hence $v(x)\ge 0$.
So for any word $u$ and letter $c\in\{0,1\}$ we have
$M(uc) v(x) = M(c) M(u) v(x) \ge M(u) v(x)\ge 0$ as $M(c)\ge I$.
Thus $x_{k+1}(x) \ge x_{k}(x)\ge 0$ and $y_{k+1}(x)\ge y_{k}(x)\ge 0$.
In conclusion, $\sigma_x^{(n)}$ and $\sigma_y^{(n)}$
are ascending sequences on $\R_+$.

Now $\phi_w(0) = \frac{[M(w)]_{12}}{[M(w)]_{22}}$.
Thus $[A v(\phi_w(0))]_2 := \frac{[A M(w)]_{22}}{[M(w)]_{22}}$ for any $A\in\R^{2\times 2}$.
So 
\begin{align*}
&x_{nm+k}(\phi_w(0))-y_{nm+k}(\phi_w(0)) \nonumber \\
&\quad= \frac{1}{[M(w)]_{22}} \left[(M((10w)^n(10w)_{1:k})-M((01w)^n(01w)_{1:k})) M(w)\right]_{22} 
\le 0  
\end{align*}
for $k=2, \dots, m$ by Claim~4 of Proposition~\ref{prop:pal}. 
So all but the first term of the sum $T_m(\phi_w(0))$ is non-positive where
\begin{align*}
T_j(x):=\sum_{k=1}^j (x_{nm+k}(x)-y_{nm+k}(x)) .
\end{align*}
Thus $T_1(\phi_w(0))\ge T_2(\phi_w(0)) \ge \dots T_{m}(\phi_w(0))$.
But 
\begin{align*}
T_m(\phi_w(0))&= 
\frac{1}{[M(w)]_{22}} \sum_{k=1}^m \left[(M((10w)^n(10w)_{1:k})-M((01w)^n(01w)_{1:k})) M(w)\right]_{22} \\
&= \frac{1}{[M(w)]_{22}} \left[S(10w) M(w(10w)^n) - S(01w) M(w(01w)^n) \right]_{22} = 0 
\end{align*}
where the last step follows from~(\ref{Ssum}).
So $T_j(\phi_w(0)) \ge 0$ for $j=1, \dots, m$. 
Yet Claims~5 and~6 of Proposition~\ref{prop:pal} give
%\begin{align*}
$\frac{d}{dx} T_j(x) = \sum_{k=1}^j [M((10w)^n(10w)_{1:k}) - M((01w)^n(01w)_{1:k})]_{21} \ge 0 .
$ %\end{align*}
So for $x\ge \phi_w(0)$ we have $T_j(x)\ge 0$ for $j=1, \dots, m$
which means that $\sigma_x^{(n)}\prec^w \sigma_y^{(n)}$.
\end{proof}

\subsection{Indexability}
\begin{theorem}
The index $\lambda^W(x)$ of~(\ref{index}) is continuous and non-decreasing for $x\in\R_+$.
\label{main}
\end{theorem}
\begin{proof}
As weight $w$ is non-negative and cost $h$ is a constant
we only need to prove the result for $\lambda(x):=\left.\lambda^W(x)\right|_{w=1,h=0}$ and we can use $w$ to denote a word.
By Proposition~\ref{mechanical}, $x\in [y_{01w},y_{10w}]$ for some mechanical word $0w1$. (Cases $x\notin (y_1,y_0)$ are clarified in the supplementary material.)

Let us show that the hypotheses of Proposition~\ref{major} are satisfied by $w$ and $x$.
Firstly, $w$ is a palindrome by Proposition~\ref{palindrome}.
Secondly, $\phi_{w01}(0) \ge 0$ and as $\phi_w(\cdot)$ is monotonically increasing, it follows that 
$\phi_{w}\circ\phi_{w01}(0)\ge\phi_w(0)$.
Equivalently, $\phi_{01w}\circ \phi_w(0) \ge \phi_w(0)$ so that $\phi_w(0)\le y_{01w}$ by Proposition~\ref{increasing}.
Hence $x \ge y_{01w} \ge \phi_w(0)$. 

Thus Proposition~\ref{major} applies, showing that the sequences $\sigma_x^{(n)}$ and $\sigma_y^{(n)}$, 
with elements $x_{nm+k}(x)$ and $y_{nm+k}(x)$ as defined in~(\ref{sigmas}), 
are non-decreasing sequences on $\R_+$ with $\sigma_x^{(n)}\prec^w \sigma_y^{(n)}$.
Also, $1/x^2$ is a symmetric function that is convex and decreasing on $\R_+$. 
Therefore Proposition~\ref{symconvex} applies giving
\begin{align}
\sum_{k=1}^m \left( \frac{\beta^{nm+k-1}}{(x_{nm+k}(x))^2}-\frac{\beta^{nm+k-1}}{(y_{nm+k}(x))^2} \right) \ge 0 \qquad \text{for any $n\in\Z_+$
where $m:=\abs{01w}$.}
\label{pos}
\end{align}
Also Proposition~\ref{mechanical} shows that the $x$-threshold orbits are
$(\phi_{u_1}(x), \dots, \phi_{u_{1:k}}(x), \dots)$ and $(\phi_{l_1}(x), \dots, \phi_{l_{1:k}}(x), \dots)$
where $u:=(01w)^\omega$ and $l:=(10w)^\omega$.
So the denominator of~(\ref{index}) is 
\begin{align*}
\sum_{k=0}^\infty \beta^k (\I{l_{k+1}=1} - \I{u_{k+1}=1}) = \sum_{k=0}^\infty \beta^{mk} (1-\beta) %= \frac{1-\beta}{1-\beta^m} \\
\Rightarrow \lambda(x) = \frac{1-\beta^m}{1-\beta} \sum_{k=1}^\infty \beta^{k-1} (\phi_{u_{1:k}}(x) - \phi_{l_{1:k}}(x)) . 
\end{align*}
Note that $\frac{d}{dx} \frac{ex+f}{gx+h} = \frac{1}{(gx+h)^2}$ for any $eh-fg=1$. Then~(\ref{pos}) gives
\begin{align*}
\frac{d \lambda(x)}{dx} 
= \frac{1-\beta^m}{1-\beta} \sum_{n=0}^\infty \sum_{k=1}^m \left( \frac{\beta^{nm+k-1}}{(x_{nm+k}(x))^2}-\frac{\beta^{nm+k-1}}{(y_{nm+k}(x))^2} \right)  
\ge 0 .
\end{align*}
But $\lambda(x)$ is continuous for $x\in\R_+$ (as shown in the supplementary material).
Therefore we conclude that $\lambda(x)$ is non-decreasing for $x\in\R_+$.
\end{proof}

\section{Further Work}
One might attempt to prove that assumption A1 holds using general results about monotone optimal policies for two-action MDPs based on submodularity~\cite{Altman95} or multimodularity~\cite{Gaujal}. 
However, we find counter-examples to the required submodularity condition.
%Such results boil down to showing that $Q(x,0)-Q(x,1)$ is non-decreasing in $x$
%where $Q(x,u)$ is the optimal cost-to-go for a fixed first action $u$. 
%However, we find numerically that the difference $Q(x,0)-Q(x,1)$ corresponding to~(\ref{singlearm}) can be {\it increasing} for small $x$, so such an approach fails. Nevertheless, the difference $Q(x,0)-Q(x,1)$ still has a single zero-crossing so that A1 still holds. 
Rather, we are optimistic that the ideas of this paper themselves
offer an alternative approach to proving A1.
It would then be natural to extend our results to settings where the underlying state evolves as $Z_{t+1} \mid \mathcal{H}_t \sim \mathcal{N}(m Z_t, 1)$ for some multiplier $m\neq 1$ and to cost functions other than the variance. Finally, the question of the indexability of the discrete-time Kalman filter in multiple dimensions remains open.
\vfill
\pagebreak

\bibliographystyle{abbrv}
{\small \bibliography{Draft1}}

\pagebreak

\section{Supplementary Material: Introduction}
The results used but not proved in the main paper are given here as:
\begin{itemize}
\item Proposition~\ref{prop:incdec} which was used to show that $\phi_w(0)\le x$,
\item Proposition~\ref{prop:zrange} for the range of $x$ giving a specific mechanical word,
\item Proposition~\ref{proposition:continuity} showing the index is continuous for $x\in\R_+,$
\item Proposition~\ref{prop:pal} showing the properties of $M(p)$ when 
$p$ is a palindrome.
\item and Proposition~\ref{prop:symconvex} for weak supermajorisation with $\beta \neq 1$.
\end{itemize}
A clarification of the extreme cases
of Theorem~1 of the main paper is presented in the final section.

\section{From $x$-Threshold Policies to Mechanical Words}
Some concepts relating to mechanical words appeared
as early as 1771 in Jean Bernoulli's study of continued fractions
(Berstel {\it et al}, 2008).
The term ``mechanical sequences'' appears in the work of 
Morse and Hedlund (Am. J. Math., Vol 62, No. 1, 1940, p. 1-42)
who had just introduced the term ``symbolic dynamics''.
Morse and Hedlund studied the concept from the perspective of 
sequences of the form $\lfloor c+k\beta\rfloor$ for $c,\beta\in \R$ and $k\in\Z$.
They also studied the concept from the perspective of differential equations, 
motivating the term ``Sturmian sequences.''
Since that time there has been tremendous progress in the study of such sequences 
from the perspective of Combinatorics on Words (Lothaire, 2001).
However, the recent (and highly-approachable) 
paper of Rajpathak, Pillai and Bandyopadhyay (Chaos, Vol. 22, 2012)
on the piecewise-linear map-with-a-gap discovers such sequences
without recognising them as mechanical sequences.
Proposition~\ref{prop:zrange} of this section is a substantial
generalisation of that result
and we could not find this proposition explicitly stated
in the literature.
Our result is not surprising if one has the intuition
that there is a topological conjugacy between the maps of this section
and the piecewise linear map-with-a-gap. 
However, it might be difficult to explicitly identify the 
appropriate topological conjugacy and thereby prove our result
for all cases considered here.

\subsection{Definitions}
Let $\pi$ denote a word consisting of a string of 0s and 1s in which the $k^{th}$ letter is $\pi_k$
and letters $i, i+1, \dots, j$ are $\pi_{i:j}$.
Let $\abs{\pi}$ be the length of $\pi$ and $\abs{\pi}_w$ for a word $w$ be the number
of times that word $w$ appears in $\pi$.
Let $\epsilon$ denote the empty word and $\pi^\omega$ denote the infinite word constructed by repeatedly concatenating $\pi$.

Consider two functions $\phi_0 : \IR \rightarrow \IR$ and $\phi_1 : \IR \rightarrow \IR$
where $\IR$ is an interval of $\mathbb{R}.$
We define the transformation $\phi_\pi : \IR \rightarrow \IR$ for any word $\pi$ by the composition
\begin{align*}
\phi_\pi(x) := 
\phi_{\pi_{\abs{\pi}}} 
\circ \cdots \circ \phi_{\pi_2} \circ \phi_{\pi_1} (x) .
\end{align*}
Let $y_\pi \in \IR$ be the fixed point of $\phi_\pi$, so $\phi_\pi(y_\pi) = y_\pi$, assuming a unique fixed point on $\IR$ exists.

Given $x\in \IR$, we call the sequence $(x_k : k \ge 1)$ the {\it $x$-threshold orbit for $\phi_0, \phi_1$} if
\begin{align*}
x_1 &= \phi_1( x), & 
x_{k+1} &= \begin{cases} \phi_1 (x_k) & \text{if $x_k \ge x$} \\ \phi_0 (x_k ) & \text{if $x_k < x$} \end{cases} & \text{for $k\ge 1$} .
\end{align*}
We call $\pi$ the {\it $x$-threshold word for $\phi_0, \phi_1$} if it is the shortest word such that $x_{k+1} = \phi_{(\pi^\omega)_k} (x_k)$ for all $k\ge 1$.
We shall just write {\it $x$-threshold orbit} and {\it $x$-threshold word} where $\phi_0, \phi_1$ are obvious from the context.

For $p\ge 1$, let $L_p, R_p$ be the morphisms (substitutions)
\begin{align*}
L_p : \begin{cases} 0 \rightarrow 0^{p+1} 1 \\ 1 \rightarrow 0^{p} 1 \end{cases}
&& R_p : \begin{cases} 0 \rightarrow 0 1^{p} \\ 1 \rightarrow 01^{p+1} \end{cases} .
\end{align*}

We say $\pi$ is a {\it valid word} if 
$\pi\in \{0, 1\}$ 
or $\pi \in \{L_p(w), R_p(w) : p\ge 1\}$ for some valid word $w$. 

{\bf Remark.} The morphisms $L_p, R_p$ generate the Christoffel tree
so {\it valid words are mechanical words}. 
To see this, note that the Christoffel tree is generated
by the following morphisms (Berstel {\it et al}, 2008, p.~37)
\newcommand{\tD}{{\tilde D}}
\begin{align*}
G &: \begin{cases} 0\rightarrow 0 \\ 1 \rightarrow 01 \end{cases}
& \tD &: \begin{cases} 0\rightarrow 01 \\ 1\rightarrow 1 \end{cases} .
\end{align*}
We may translate (from English to French) as
$L_p = G^p\circ \tilde D$ and $R_p = \tilde D^p \circ G$
so any composition of $L_p$ and $R_p$ can be written as a composition of 
$G$ and $\tilde D$.
Likewise, any composition of $G$ and $\tilde D$ can be written as a composition
of $L_p$ and $R_p$. Specifically if $p_k, q_k, p_{k+1} \ge 2$ then
\begin{align*}
&\cdots \circ G^{p_k-1} \circ \tD^{q_k} \circ G^{p_{k+1}} \circ \tD \circ \cdots
\\
&\quad= \cdots \circ (G^{p_k-1}\circ \tD) \circ (\tD^{q_k-1}\circ G) \circ
(G^{p_{k+1}-1}\circ\tD) \circ \cdots\\
&\quad=\cdots \circ L_{p_k-1}\circ R_{q_k-1} \circ L_{p_{k+1}-1} \circ \cdots
\end{align*}
whereas if $q_k = 1$ we have
\begin{align*}
&\cdots \circ G^{p_k-1} \circ \tD \circ G^{p_{k+1}} \circ \tD \circ \cdots
\\
&\quad= \cdots \circ (G^{p_k-1}\circ \tD) \circ
(G^{p_{k+1}}\circ\tD) \circ \cdots\\
&\quad=\cdots \circ L_{p_k-1}\circ L_{p_{k+1}} \circ \cdots .
\end{align*}
A symmetric argument holds if $p_k = 1$ or $p_{k+1} = 1$.

\subsection{Fixed Points}
Throughout, we make the following assumption about $\phi_0, \phi_1$.
The existence of fixed points $y_0, y_1$ is addressed immediately thereafter.

{\bf Assumption A2.} {\it Functions $\phi_0 : \IR \rightarrow \IR, \phi_1 : \IR \rightarrow \IR$, where $\IR$ is an interval of $\mathbb{R}$, are increasing and non-expansive. 
Equivalently, for all $x, y\in\IR : x < y$ and for $k \in \{0,1\}$ we have
\begin{align*}
\underbrace{\phi_k(x) < \phi_k(y)}_{\text{increasing}} \qquad \qquad \text{and} \qquad \qquad \underbrace{\phi_k(y) - \phi_k(x) < y-x}_{\text{non-expansive}} .
\end{align*}
Furthermore, the fixed points $y_0, y_1$ of $\phi_0, \phi_1$ satisfy $y_1 < y_0$. }

\begin{proposition}
Suppose A2 holds, that $x\in\IR$ and that $w$ is any non-empty word. Then 
$\phi_w(x)$ is increasing and non-expansive.
Further, the fixed point $y_w$ exists and is unique. 
\label{prop:monexist}
\end{proposition}

\begin{proof}
First we show that $\phi_w(x)$ is increasing, by induction.
In the base case, $\abs{w} = 1$ and the claim follows from A2.
For the inductive step assume $\phi_u(x)$ is increasing, 
where $w = au$ for some $a\in \{0,1\}$ and word $u$.
Then for any $x,y\in\IR : x < y$, 
\begin{align*}
\phi_w(y) &= \phi_u(\phi_a(y)) \\
&> \phi_u(\phi_a(x)) & \text{as $\phi_a(y) > \phi_a(x)$ and $\phi_u$ is increasing} \\
&= \phi_w(x) .
\end{align*}
Therefore $\phi_w$ is increasing.

Now we show that $\phi_w(x)$ is non-expansive, by induction.
If $\abs{w} = 1$ then this follows from A2.
Else, say $\phi_u(x)$ is non-expansive where $w = ua$ and $a\in \{0,1\}$.
Then for any $x,y\in\IR : x < y$, 
\begin{align*}
\phi_w(y) - \phi_w(x) &= \phi_a(\phi_u(y)) - \phi_a(\phi_u(x)) \\
&< \phi_u(y) - \phi_u(x) & \text{as $\phi_u(y) > \phi_u(x)$ and $\phi_a$ is non-expansive}\\
&< y-x & \text{as $\phi_u$ is non-expansive.}
\end{align*}
Therefore $\phi_w$ is non-expansive.

Let $\psi(x) := \max\{\phi_0(x), \phi_1(x)\}$.
As $\phi_1$ is non-expansive we have
\begin{align*}
y_1 = \phi_1(y_1) > \phi_1(y_0) + y_1 - y_0
\end{align*}
which rearranges to give $\phi_1(y_0) < y_0$, 
so that $\psi(y_0) = y_0$.
Also $\psi$ is increasing as $\phi_0, \phi_1$ are increasing,
so $\phi_w(y_0) \le \psi^{(\abs{w})}(y_0) = y_0$.

We now prove that $y_w$ exists.
The argument of the previous paragraph shows that 
$g(x) := x - \phi_w(x)$ satisfies $g(y_0) \ge 0$.
A symmetric argument leads to the conclusion that $g(y_1) \le 0$.
Clearly $g(x)$ is a continuous function, so by the intermediate value theorem, 
there is some $y \in [y_0, y_1]$ for which $g(y) = 0$. 
Equivalently $y = \phi_w(y)$. 
Therefore a fixed point $y_w$ exists.

To show that the fixed point is unique, suppose both $y$ and $z$ are fixed points with $y > z$.
As $\phi_w$ is non-expansive we have $\frac{\phi_w(y) - \phi_w(z)}{y-z} < 1$. 
Yet, as $\phi_w(y) = y, \phi_w(z) = z$ we have
\begin{align*}
\frac{\phi_w(y) - \phi_w(z)}{y-z} = 1 .
\end{align*}
This is a contradiction. 
Therefore the fixed point is unique.
\end{proof}

Given a word $w$, the next proposition shows when the transformation $\phi_w$ increases or decreases its argument and what might be deduced from such an increase or decrease.

\begin{proposition} Suppose A2 holds, $x\in \IR$ and $w$ is any non-empty word. Then 
\begin{align*}
x < \phi_w(x) \ \Leftrightarrow \  \phi_w(x) < y_w \ \Leftrightarrow \  x < y_w
&& \text{and} &&
x > \phi_w(x) \ \Leftrightarrow \  \phi_w(x) > y_w \ \Leftrightarrow \  x > y_w.
\end{align*}
\label{prop:incdec}
\end{proposition}
\begin{proof}
We use Proposition~\ref{prop:monexist} throughout the argument without further mention.

Say $x < y_w$. As $\phi_w$ is increasing,
\begin{align*}
\phi_w(x) &< \phi_w(y_w) = y_w  
\end{align*}
where the equality is the definition of $y_w$. Also, as $\phi_w$ is non-expansive,
\begin{align*}
y_w &= \phi_w(y_w) < \phi_w(x) + y_w - x 
\end{align*}
which rearranges to give $x < \phi_w(x)$.

Now say $x > y_w$. As above, we then have $\phi_w(x) > \phi_w(y_w) = y_w$
and 
\begin{align*}
y_w = \phi_w(y_w) > \phi_w(x) + y_w - x
\end{align*}
so that $x > \phi_w(x)$.

The contrapositive of $x > y_w \Rightarrow \phi_w(x) > y_w$ is
$\phi_w(x) \le y_w \Rightarrow x \le y_w$.
But if $\phi_w(x) \neq y_w$ then $x \neq y_w$ as $\phi_w$ is increasing and therefore injective. 
Thus $\phi_w(x) < y_w \Rightarrow x < y_w$.

The contrapositive of $x > y_w \Rightarrow x > \phi_w(x)$ is
$x \le \phi_w(x) \Rightarrow x \le y_w$.
But if $x \ne \phi_w(x)$ then $x \ne y_w$ as $y_w$ is a fixed point.
So we can conclude that $x < \phi_w(x) \Rightarrow x < y_w$.

By symmetry, $\phi_w(x) > y_w \Rightarrow x > y_w$ and $x > \phi_w(x) \Rightarrow x > y_w$.
This completes the proof.
\end{proof}

\begin{proposition}
Suppose A2 holds and $\pi$ is any word satisfying $\abs{\pi}_0 \abs{\pi}_1>0$. Then $y_1 < y_\pi < y_0$.
\label{prop:ypi}
\end{proposition}
\begin{proof}
Say $y_\pi\le y_1$. As $\abs{\pi}_0 >0$ we can write $\pi =: s01^q$ for some $q\ge 0$. Thus
\begin{align*}
y_\pi = \phi_\pi(y_\pi) &\le \phi_{s01^q}(y_1) & \text{as $\phi_\pi$ is increasing} \\
&= \phi_{s0}(y_1) & \text{as $\phi_\epsilon(y_1) = \phi_1(y_1) = y_1$} \nonumber \\
&> \phi_s(y_1) & \text{by Proposition~\ref{prop:incdec}} \nonumber \\
&\ge y_1 & \text{by repeating the same argument if $\abs{s}_0 > 0$.}
\end{align*}
But this contradicts $y_\pi \le y_1$. Therefore $y_\pi > y_1$.

A symmetrical argument leads to the conclusion that $y_\pi < y_0$. 
\end{proof}

\begin{proposition}
If A2 holds and $n\ge 1$ then
$y_{10^{n-1}} < y_{010^{n-1}} < y_{10^n}$ and $y_{01^n} < y_{101^{n-1}} < y_{01^{n-1}}.$
\label{prop:yorder}
\end{proposition}
\begin{proof}
As $y_{10^{n-1}} < y_0$ by Proposition~\ref{prop:ypi} we have
$\phi_0(y_{10^{n-1}}) > y_{10^{n-1}}$ by Proposition~\ref{prop:incdec} so that
\begin{align*}
\phi_{010^{n-1}}(y_{10^{n-1}}) = \phi_{10^{n-1}} (\phi_0(y_{10^{n-1}})) > \phi_{10^{n-1}}(y_{10^{n-1}}) = y_{10^{n-1}}
\end{align*}
so Proposition~\ref{prop:incdec} gives $y_{010^{n-1}} > y_{10^{n-1}}.$

Furthermore $y_{10^n} = \phi_0(y_{010^{n-1}})$ by definition of $y_\pi$
and $y_{010^{n-1}} < y_0$ by Proposition~\ref{prop:ypi} so that
$\phi_0(y_{010^{n-1}}) > y_{010^{n-1}}$ by Proposition~\ref{prop:incdec}.
Thus $y_{10^{n}} > y_{010^{n-1}}$.

The proof that $y_{01^n} < y_{101^{n-1}} < y_{01^{n-1}}$ is symmetrical.
\end{proof}

\begin{proposition}
Suppose A2 holds, $M \in \{L_q, R_q : q\ge 1\}$ and $\tilde w$ is any word. 
Let $\tilde y_v$ be the fixed point of $\tilde\phi_v := \phi_{M(v)}$ for any word $v$
and let $0w1 := M(0\tilde w 1)$.
Then
\begin{align*}
\tilde x \in [\tilde y_{01\tilde w}, \tilde y_{10\tilde w}] \ \Leftrightarrow \
x := \phi_{0^q}(\tilde x) \in [y_{01w}, y_{10w}].
\end{align*}
\label{prop:zzt}
\end{proposition}
\begin{proof}
Say $M = L_q$. Note that
\begin{align*}
\phi_{0^q}(\tilde y_{01\tilde w}) 
&= \phi_{0^q}(y_{L_q(01\tilde w)}) & \text{as $\tilde y_v$ is the fixed point of $\tilde\phi_{v} = \phi_{L_q(v)}$} \\
&= \phi_{0^q}(y_{0^q 01L_q(1\tilde w)}) & \text{as $L_q(0) = 0^{q}01$} \\
&= y_{01 L_q(1\tilde w) 0^q} & \text{as $\phi_a(y_{ab}) = y_{ba}$ for any words $a,b$} \\
&= y_{01 w} & \text{as $0w1 = L_q(0\tilde w 1) = 0 L_q(1 \tilde w) 0^q 1$} 
\intertext{and}
\phi_{0^q}(\tilde y_{10\tilde w}) 
&= \phi_{0^q}(y_{L_q(10\tilde w)}) \\
&= \phi_{0^q}(y_{0^q1 L_q(0\tilde w)}) \\
&= y_{1 L_q(0\tilde w) 0^q} \\
&= y_{10w} & \text{as $0w1 = L_q(0\tilde w) 0^q 1$} .
\end{align*}

Proposition~\ref{prop:monexist} shows that $\tilde y_{01\tilde w}, \tilde y_{10\tilde w}$
exist. 
So the above equalities show that an inverse $\phi^{(-1)}_{0^q}(x)$ exists for $x \in \{ y_{01w}, y_{10w} \}$. 
As $\phi_{0^q}$ is increasing and continuous, we have
\begin{align*}
x\in [y_{01w}, y_{10w}] \ \Leftrightarrow \ \tilde x \in [\phi^{(-1)}_{0^q}(y_{01w}), \phi^{(-1)}_{0^q}((y_{10w})] = [\tilde y_{01\tilde w}, \tilde y_{10\tilde w}] .
\end{align*}

The proof for $M = R_q$ is symmetric.
\end{proof}

\subsection{$x$-Threshold Words}
\begin{proposition}
Suppose A2 holds, $\pi$ is the $x$-threshold word and $n\ge 1$. Then
\begin{enumerate}
\item $x \le y_{10^{n-1}} \Rightarrow \abs{\pi^\omega}_{0^n} = 0$
\item $x \ge y_{010^{n-1}} \Rightarrow \abs{\pi^\omega}_{10^{n-1}1} = 0$
\item $x \ge y_{01^{n-1}} \Rightarrow \abs{\pi^\omega}_{1^n} = 0$
\item $x \le y_{101^{n-1}} \Rightarrow \abs{\pi^\omega}_{01^{n-1}0} = 0$
\end{enumerate}
\label{prop:0n}
\end{proposition}
\begin{proof}
If $x \le y_1$ then it follows from Proposition~\ref{prop:incdec} that the $x$-threshold word is $\pi = 1$. 
Likewise if $x > y_0$ then the $x$-threshold word is $\pi = 0$.
In these cases Claims 1 and 2 hold, so in the following we assume that $y_1 < x \le y_0$.

{\it Claim 1:}
Let $(x_k)$ the $x$-threshold orbit.
If $(\pi^\omega)_{k:k+n-2} = 0^{n-1}$ for some $k$, then 
\begin{align*}
x_{k+n-1} &= \phi_{0^{n-1}}(x_k) & \text{by definition of $(x_k)$} \\
&\ge \phi_{0^{n-1}}(\phi_1(x)) & \text{as $x_k \ge \phi_1(x)$ for all $k\ge 0$ and $\phi_{0^{n-1}}$ is increasing} \\
&= \phi_{10^{n-1}}(x) \\
&\ge x & \text{if $x \le y_{10^{n-1}}$ by Proposition~\ref{prop:incdec}.}
\end{align*}
But if $x_{k+n-1} \ge x$ then $\pi_{k+n-1} = 1$ by definition $\pi$.
Therefore $\abs{\pi}_{0^n} = 0.$
 
{\it Claim 2:}
Let $(x_k)$ be the $x$-threshold orbit.
If $(\pi^\omega)_{k:k+n-1} = 10^{n-1}$ for some $k$, then
\begin{align*}
x_{k+n} &= \phi_{10^{n-1}}(x_k) \\
&< \phi_{10^{n-1}}(\phi_0(x)) & \text{as $x_k < \phi_0(x)$ for all $k\ge 0$ and $\phi_{10^{n-1}}$ is increasing} \\
&= \phi_{010^{n-1}}(x) \\
&\le x & \text{if $x \ge y_{010^{n-1}}$ by Proposition~\ref{prop:incdec}.} 
\end{align*}
But if $x_{k+n} < x$ then $(\pi^\omega)_{k+n} = 0$.
Therefore $\abs{\pi}_{10^{n-1}1} = 0.$

The proof of Claims~3 and~4 is symmetrical.
\end{proof}

\begin{proposition}
Suppose A2 holds and $\pi$ is a $x$-threshold word. Then
\begin{enumerate}
\item $\abs{\pi}_{00} > 0 \Rightarrow \pi = L_n(w)$ for some word $w$ and some $n\ge 1$
\item $\abs{\pi}_{11} > 0 \Rightarrow \pi = R_n(w)$ for some word $w$ and some $n \ge 1$
\end{enumerate}
\label{prop:0n1}
\end{proposition}

\begin{proof}
First, applying Claims 1 and 3 of Proposition~\ref{prop:0n} with $n = 2$ we have
$\abs{\pi}_{00} = 0$ for $x \le y_{10}$ and $\abs{\pi}_{11} = 0$ for $x \ge y_{01}$.
Furthermore $y_{10} = \phi_{0}(y_{01}) > y_{01}$ by Proposition~\ref{prop:incdec}.
Thus $\pi$ cannot contain both 00 and 11.

So, if $\abs{\pi}_{00} > 0$ then $\pi$ is of the form $0^{q_1} 1 0^{q_2} 1 \dots$
with strings of 0s separated by individual 1s.
Let $q := \min_k q_k$. 
By Propositions~\ref{prop:yorder} and~\ref{prop:0n},
$I_q := (y_{10^{q-1}}, y_{010^{q}})$ is the only set of $x$ values for which
$\pi^\omega$ can contain $10^q1$.
Thus $\pi^\omega$ can only contain both $10^q1$ and $10^{q+1}1$ in the interval
\begin{align*}
F_q := I_q \cap I_{q+1} = (y_{10^{q-1}}, y_{010^{q}}) \cap (y_{10^{q}}, y_{010^{q+1}}) 
=  (y_{10^{q}}, y_{010^{q}} ) 
\end{align*}
noting Proposition~\ref{prop:yorder} gives
$y_{10^{q-1}} < y_{010^{q-1}} < y_{10^q} < y_{010^q}.$

Finally, we have $F_q \cap F_{q'} = \emptyset$ for $q \neq q'$, which also follows
from Proposition~\ref{prop:yorder}.
Thus if $\abs{\pi}_{00} > 0$ then $\pi$ is a concatenation of $L_q(0)$ and $L_q(1)$.
Equivalently $\pi = L_q(w)$ for some word $w$ and some $q \ge 1$ as in Claim 1.

The proof of Claim 2 is symmetric.
\end{proof}

\begin{proposition}
Suppose A2 holds and $\pi$ is a $x$-threshold word. Then $\pi$ is a valid word.
\label{prop:valid}
\end{proposition}
\begin{proof}
There are three cases to consider: either $\abs{\pi}_{00} = \abs{\pi}_{11} = 0$ or $\abs{\pi}_{00} > 0$ or $\abs{\pi}_{11} > 0$.

{\it First case:} The only non-empty words not containing $00$ or $11$ are $0, 1, (01)^n, (10)^n$ for some $n\ge 1$.
Now $x$-threshold words start with 0 unless $x\le y_1$ (in which case $\pi = 1$) so $\pi \ne (10)^n$.
Further, the $x$-threshold word was defined to be the shortest word such that such that $x_{k+1} = A_{(\pi^\omega)_k} x_k$ so this leaves us with the options $0, 1, 01$.
These are all valid words.

{\it Second case:} If $\pi$ contains 00, we may write $\pi = L_q(w)$ for some word $w$, by Proposition~\ref{prop:0n1}.
Now from point $x_k$ on the $x$-threshold orbit
we have $\pi_{k:k+q} = 0^{q+1}$ if and only if $\phi_{0^q} (x_k) < x$ which corresponds to 
$x_k < \phi_0^{(-q)}(x) =: \tilde x$.
So the word $w$ corresponds to a $\tilde x$-threshold orbit $(\tilde x_k : k \ge 1)$ 
for $\psi_0(x) := \phi_{0^{q+1} 1}(x), \psi_1(x) := \phi_{0^q 1}(x)$.
To spell it out, we have 
\begin{align*}
\tilde x_1 &= \psi_1( \tilde x), & \tilde x_{k+1} &= \psi_{w_k}( \tilde x_k), & 
w_k &= \begin{cases} 1 & \text{if $\tilde x_k \ge \tilde x$} \\ 0 & \text{if $\tilde x_k < \tilde x$} \end{cases}
& \text{for $k\ge 1$}
\end{align*}
and as for the original system, we define $\tilde y_\pi$ as the fixed point $\tilde y_{\pi} = \psi_{\pi}(\tilde y_{\pi})$.

Now $\psi_0, \psi_1$ are non-negative, as $\phi_0, \phi_1$ are non-negative.
Also $\psi_0, \psi_1$ are monotonically increasing and non-expansive by Proposition~\ref{prop:monexist}.
Further, 
\begin{align*}
\phi_{0^{q+1} 1}(y_{0^q 1}) = \phi_{0^q 1}(\phi_0(y_{0^q1})) > \phi_{0^q1}(y_{0^q1}) = y_{0^q1} 
\end{align*}
so that $y_{0^{q+1}1} > y_{0^q 1}$ by Proposition~\ref{prop:incdec}.
But by definition $\tilde y_0 = y_{0^{q+1} 1}$ and $\tilde y_0 = y_{0^q1}$, 
so that $\tilde y_1 < \tilde y_0$.
Therefore $\psi_0, \psi_1$ satisfy A2. 

{\it Third case:} We prove that $\pi = R_q(w)$ for some positive integer $q$ and word $w$.
We also show that word $w$ is a $\hat x$-threshold word for a pair of functions
(say) $\chi_0, \chi_1$ which satisfy A2.
The argument is symmetric to the second case, so it is omitted.

In conclusion, either 
\begin{enumerate}
\item $\pi\in\{0, 1, L_1(1)\}$ which are valid words
\item $\pi = L_q(w)$ where $w$ is a $\tilde x$-threshold word for $\psi_0, \psi_1$ which satisfy Propositions~\ref{prop:monexist}-\ref{prop:0n1} and therefore $w$ satisfies this conclusion
\item or $\pi = R_q(w)$ where $w$ is a $\hat x$-threshold word
for $\chi_0, \chi_1$ which satisfy Propositions~\ref{prop:monexist}-\ref{prop:0n1} and therefore $w$ satisfies this conclusion.
\end{enumerate}
Thus $\pi$ is a valid word. This completes the proof.
\end{proof}

The following proposition shows that all valid words are $x$-threshold words 
and tells us explicitly which values of $x$ produce a given valid word.
It is one of the key results of the main paper.
\begin{proposition}
Suppose A2 is satisfied and $0w1$ is any valid word. 
Then 
\begin{align*}
\text{$0w1$ is the $x$-threshold word} \ \Leftrightarrow \ x \in [y_{01w}, y_{10w}].
\end{align*}
\label{prop:zrange}
\end{proposition}
\begin{proof}
Let $V_1 := \{L_q(1), R_q(1) : q \ge 1\}, V_{n+1} := \{L_q(v), R_q(v) : v\in V_n, q\ge 1\}$.
Note that $V_1$ contains
$L_q(0) = 0^{q+1} 1 = L_{q+1}(1)$ 
and $R_q(0) = 01^{q}$ 
which for $q \ge 2$ equals $R_{q-1}(1)$ and for $q = 1$ equals $01 = L_1(1)$.
Thus $\cup_{n=1}^\infty V_n$ is the set of all valid words of form $0w1$.

We use induction with hypothesis
\begin{align*}
H_{n} : \quad 0w1\in V_n \ \text{is the $x$-threshold word} \ \Leftrightarrow \ x \in [y_{01w}, y_{10w}]
\end{align*}

{\it Base case ($H_1$).} 
Say $0w1 = 0^q 1$ is the $x$-threshold word. Then 
\begin{align*}
x &> \phi_{(10^q)^n 10^{q-1}}(x) &\text{for all $n\ge 0$} \\
&= \phi_{(010^{q-1})^n} (\phi_{10^{q-1}}(x)) \\
\Rightarrow \ x &\ge \lim_{n\rightarrow \infty} \phi_{(010^{q-1})^n} (\phi_{10^{q-1}}(x)) = y_{010^{q-1}} .
\end{align*}
The definition of the $x$-threshold word also gives $x \le \phi_{10^q}(x)$.
Therefore $x \ge y_{10^q}$ by Proposition~\ref{prop:incdec}.
Thus if $0^q1$ is the $x$-threshold word then $x \in [y_{01w}, y_{10w}]$.

Now say $x\in [y_{010^{q-1}}, y_{10^q}]$. Proposition~\ref{prop:ypi} gives $y_0 < x < y_1$
so that the $x$-threshold orbit $(x_k)$ is contained in $(y_0, y_1)$.
So Proposition~\ref{prop:incdec} shows that $\phi_0(x_k) > x_k$ 
and $\phi_1(x_k) < x_k$ for all $k \ge 0$.
So to prove that the $x$-threshold word is $0^q 1$ we need only show that $\phi_{(10^q)^n 10^{q-1}}(x) < x$
and $\phi_{(10^q)^n}(x) \ge x$ for all $n\ge 0$.
But if $x \ge y_{010^{q-1}}$ then for all $n\ge 0$ 
\begin{align*}
x &\ge \phi_{(010^{q-1})^n}(x) 
& \text{by Proposition~\ref{prop:incdec}} \\
&> \phi_{(010^{q-1})^n}(\phi_{10^{q-1}}(x)) 
& \text{as $y_{10^{q-1}} < y_{010^{q-1}} \le x$ by Claim~3 of Proposition~\ref{prop:yorder}} \\
&= \phi_{(10^q)^n 10^{q-1}}(x) .
\end{align*}
Also if $x \le y_{10^q}$ then $\phi_{(10^q)^n}(x) \ge x$ for all $n\ge 0$ 
by Proposition~\ref{prop:incdec}.
Therefore for $0w1 = 0^q1$, we have $x\in [y_{01w}, y_{10w}]$
implies that $0w1$ is the $x$-threshold word.

For $0w1 = 01^q$, the proof that $\pi = 01^{q} \Leftrightarrow x \in [y_{01w}, y_{10w}]$ is symmetric, so it is omitted.

{\it Inductive Step.} 
Assume $0\tilde w 1$ satisfies $H_n$.

Say $0w1 = L_q(0\tilde w 1)$. 
Let $k_i := \abs{L_q( ((0\tilde w 1)^\omega)_{1:i-1})}+1$ so $(\pi^\omega)_{k_i}$ is aligned with the start of the $i^{th}$ letter of $(0\tilde w 1)^\omega$.
Let $x_k := \phi_{((10w)^\omega)_{1:k}}(x), \tilde x_i := x_{k_i}, x = \phi_{0^q}(\tilde x) $
and let $\tilde y_v$ denote the fixed point of $\tilde\phi_v := \phi_{L_q(v)}$ for any word $v$. 
Then we have
\begin{align*}
& \text{$L_q(0\tilde w 1)$ is the $x$-threshold word for $\phi_0, \phi_1$} \\
\Leftrightarrow\qquad & \text{$((0w1)^\omega)_{k_i : k_i+q} = 0^{q+1}$ if and only if $\phi_{0^q}(x_{k_i}) < x$} \\
\Leftrightarrow\qquad & \text{$((0\tilde w 1)^\omega)_i = 0$ if and only if $\tilde x_i < \tilde x$} \\
\Leftrightarrow\qquad & \text{$0\tilde w 1$ is the $\tilde x$-threshold word for $\tilde\phi_0, \tilde \phi_1$} \\
\Leftrightarrow\qquad & \text{$\tilde x \in [\tilde y_{01\tilde w}, \tilde y_{10\tilde w}]$ as $0\tilde w 1$ satisfies $H_n$} \\
\Leftrightarrow\qquad & \text{$x \in [y_{01w}, y_{10w}]$ by Proposition~\ref{prop:zzt}} 
\end{align*}

Symmetrically we may conclude that $\pi = 0w1 = R_q(0\tilde w 1) \Leftrightarrow x\in  [y_{01w}, y_{10w}]$. Therefore $H_{n+1}$ is true.

This completes the proof.
\end{proof}

\section{Continuity of the Index}
We showed that the Whittle index is increasing on the domain of each fixed Christoffel word.
However, we also need to show that the index is continuous as we move between words.
So here we prove the following proposition.

\begin{proposition}
Suppose $\lambda(\cdot)$ is as in the main paper. Then $\lambda(x)$ is a continuous function of $x\in\R_+$.
\label{proposition:continuity}
\end{proposition}

We use the following definitions.

{\bf Definition.} Let $\tilde w$ be the reverse of word $w$, $w^\omega$ be the word constructed by concatenating $w$ infinitely many times, $\abs{w}$ be the length of word $w$ and $\abs{w}_u$ be the number of times that word $u$ is a factor of $w$.

{\bf Definition.} For a possibly-infinite word $w$ and numbers $x\in\R, \beta\in (0,1)$ define
\begin{align*}
S(w,x) &:= \sum_{n=0}^{\abs{w}-1} \beta^n \phi_{w_{1:n}}(x) \\
\lambda(0w1, x) &:= \frac{1 - \beta^{\abs{0w1}}}{1 - \beta} \left( S((01w)^\omega,x) - S((10w)^\omega,x) \right) .
\end{align*}

{\bf Remark.} If $\pi$ is the $x$-threshold word then $\lambda(x) = \lambda(\pi,x)$ where $\lambda(x)$ is the Whittle index.

{\bf Remark.} For a word $ab$, this definition gives
\begin{align}
S(ab,x) &= S(a,x) + \beta^{\abs{a}} S(b,\phi_a(x)) 
\label{eq:Sab}
\intertext{
so for $\abs{\phi_{a^\omega}(x)} < \infty$ and $\beta\in (0,1)$ we have}
S(a^\omega b, x) &= S(a^\omega,x) .
\label{eq:SaInf}
\intertext{
Further, if $x_a = \phi_a(x_a)$ then the formula for the sum of a geometric progression gives}
S(a^\omega ,x_a) &= \frac{S(a,x_a)}{1 - \beta^{\abs{a}}} .
\label{eq:SaOmega}
\end{align}

{\bf Definition.} Let $X_\pi$ be the range of $x$ for which the $x$-threshold word is $\pi$.

The following construction is closely related to the beautiful {\it Christoffel tree} (Berstel {\it et al}, 2008). 
%cite
%\begin{verbatim}
%@book{berstel2009combinatorics,
%  title={Combinatorics on Words: 
%          Christoffel words and repetitions in words},
%  author={Berstel, Jean and Lauve, Aaron and 
%          Reutenauer, Christophe and Saliola, Franco},
%  volume={27},
%  year={2009},
%  publisher={American Mathematical Soc.}
%}
%\end{verbatim}

{\bf Definition.} Consider the mapping $C$ which takes a sequence of words and returns a sequence
containing the original words mingled with the concatenation of neighbouring words
as follows:
\begin{align*}
C((a, b, c, d, \dots, x, y, z)) := (a, ab, b, bc, c, cd, d, \dots, x, xy, y, yz, z) .
\end{align*}
Now consider the sequences $t_k := C^{(k)}((0,1))$ for $k\ge 0$. The first few such sequences are 
\begin{align*}
t_0 &= (0, &&&&&&& &&&&&&&& 1) \\
t_1 &= (0, &&&&&&& 01, &&&&&&&& 1) \\
t_2 &= (0, &&& 001, &&&& 01, &&&& 011, &&&& 1) \\
t_3 &= (0, & 0001, && 001, && 00101, && 01, && 01011, && 011, && 0111, && 1) &.
\end{align*}

{\bf Remark.} If $u\in t_k$ then $\abs{u} \ge 1$ for any $k\ge 0$.
Now suppose $u, v$ are adjacent in $t_k$ and we have $\abs{uv} \ge k+2$.
Then $t_{k+1}$ contains $u, uv, v$ from which we can construct $uuv$ and $uvv$.
But $\abs{uuv} = \abs{u} + \abs{uv} \ge 1 + k+2 = k+3$ and 
$\abs{uvv} = \abs{uv} + \abs{v} \ge k+2 + 1 = k+3$.
Thus, by induction, we have shown that
\begin{align}
\abs{uv} &\ge k+2 & \text{for any adjacent pair $u,v$ in $t_k$ and any $k\ge 0$.}
\label{remark:uv}
\end{align}

\subsection{Long Common Prefixes}
We gather the results needed to prove Proposition~\ref{proposition:continuity}.
Most of these results these relate to the notion that if $\abs{x-y}$ is small and 
$a,b$ are the $x$- and $y$-threshold words, 
then words $a,b$ usually have a long common prefix, 
although this is not always the case.

The following simple result is repeatedly used in the other Lemmas of this subsection.
\begin{lemma}
Suppose $(0a1, 0b1)$ is a standard pair. Then $a10b = b01a$.
\label{lemma:a10b}
\end{lemma}

\begin{proof}
As $(0a1,0b1)$ is a standard pair, $0a10b1 =: 0w1$ is a Christoffel word.
As $0a1,0b1,0w1$ are Christoffel words, $a,b,w$ are palindromes.
Thus $a10b = w = \tilde w = \tilde b 01 \tilde a = b01a.$
\end{proof}

If $(0a1,0b1)$ is a standard pair, then the interval $X_{0b1}$ is immediately
to the left of $X_{0a1(0b1)^\omega}$.
Since the words $0b1$ and $0a1 (0b1)^\omega$ can differ within the first few letters, continuity of $\lambda(x)$ at $x = \sup X_{0b1}$ is not obvious.
Similarly, $X_{(0a1)^\omega 0b1}$ is immediately to the left of $X_{0a1}$.
However, the factors $1 - \beta^{\abs{(0a1)^\omega 0b1}}$
and $1 - \beta^{\abs{0a1}}$ appearing in the definitions of the corresponding Whittle indices
are different for $\abs{a} < \infty$.
Thus continuity of $\lambda(x)$ at $x = \sup X_{0a1}$ is not obvious.
The next two Lemmas address these questions.

\begin{lemma}
Suppose $(0a1,0b1)$ is a standard pair and let $x = \phi_{10b}(x)$. Then
\begin{align*}
\lambda({0b1} ,x)  = \lambda({0a1(0b1)^\omega},x).
\end{align*}
\label{lemma:leftSide}
\end{lemma}
\begin{proof}
The right-hand side $\lambda({0a1(0b1)^\omega},x)$ involves the sum
\begin{align}
S(10a1(0b1)^\omega, x) 
&= S(10b01a(10b)^\omega, x) 
& \text{by Lemma~\ref{lemma:a10b}} \nonumber \\
&= S(10b,x) + \beta^{\abs{10b}} S(01a(10b)^\omega, \phi_{10b}(x)) &\text{by~\ref{eq:Sab}} \nonumber \\
&= S(10b,x) + \beta^{\abs{10b}} S(01a(10b)^\omega, x)  
& \text{as $x = \phi_{10b}(x)$} \nonumber \\
&= (1-\beta^{\abs{10b}}) S((10b)^\omega, x) + \beta^{\abs{10b}} S(01a(10b)^\omega, x)  & \text{by~\ref{eq:SaOmega}} 
\label{eq:Slast} .
\end{align}

Now we note that repeated application of Lemma~\ref{lemma:a10b} gives
\begin{align}
01a(10b)^\omega = 01 a10b (10b)^\omega = 01b \, 01a (10b)^\omega =(01b)^\omega 01 a .
\label{eq:a10bRepeated}
\end{align}

Thus
\begin{align*}
\lambda({0a1(0b1)^\omega},x) &= \frac{1 - \beta^{\abs{0a1(0b1)^\omega}}}{1-\beta}
\left( S((01a1(0b1)^\omega)^\omega,x) - S((10a1(0b1)^\omega)^\omega,x) \right) \\
&= \frac{S(01a1(0b1)^\omega,x) - S(10a1(0b1)^\omega,x)}{1-\beta} &\text{by~\ref{eq:SaInf}} \\
&= \frac{1 - \beta^{\abs{10b}}}{1-\beta} \left( S(01a(10b)^\omega,x) -  S((10b)^\omega, x) \right) & \text{by~\ref{eq:Slast}}\\
&= \frac{1 - \beta^{\abs{10b}}}{1-\beta} \left( S((01b)^\omega,x) -  S((10b)^\omega, x) \right) & \text{by~\ref{eq:a10bRepeated}} \\
&= \lambda({0b1},x) .
\end{align*}
This completes the proof.
\end{proof}

\begin{lemma}
Suppose $(0a1,0b1)$ is a standard pair and let $x = \phi_{01a}(x)$. Then
\begin{align*}
\lambda({(0a1)^\omega 0b1},x) = \lambda({0a1},x) .
\end{align*}
\label{lemma:rightSide}
\end{lemma}
\begin{proof}
The left-hand side $\lambda({(0a1)^\omega 0b1},x)$ involves the sum
\begin{align}
S(01(a10)^\omega 0b1,x) &= S(01(a10)^\omega,x) 
&\text{by~\ref{eq:SaInf}} \nonumber \\
&= S(01a,x) + \beta^{\abs{01a}} S((10a)^\omega,\phi_{01a}(x)) 
&\text{by~\ref{eq:Sab}} \nonumber \\
&=S(01a,x) + \beta^{\abs{01a}} S((10a)^\omega,x) 
&\text{as $x = \phi_{01a}(x)$} \nonumber \\
&= (1-\beta^{\abs{01a}}) S((01a)^\omega,x) + \beta^{\abs{01a}} S((10a)^\omega,x) 
&\text{by~\ref{eq:SaOmega}} .
\label{eq:lastSa01}
\end{align}

Thus
\begin{align*}
\lambda({(0a1)^\omega 0b1},x) 
&= \frac{1-\beta^{\abs{(0a1)^\omega 0b1}}}{1-\beta} 
\left( S((01(a10)^\omega 0b1)^\omega,x) - S((10(a10)^\omega 0b1)^\omega,x) \right) \\
&= \frac{1}{1-\beta} 
\left( S(01(a10)^\omega 0b1,x) - S((10a)^\omega ,x) \right) 
&\text{by~\ref{eq:SaInf}}\\
&= \frac{1-\beta^{\abs{01a}}}{1-\beta} (S((01a)^\omega,x) - S((10a)^\omega,x) ) & \text{by~\ref{eq:lastSa01}}\\
&= \lambda({0a1},x) .
\end{align*}
This completes the proof.
\end{proof}

To demonstrate continuity at other points, we will need to rely on the fact that
nearby words often have a long common prefix as shown by the following two Lemmas.

\begin{lemma}
Suppose $(0a1, 0b1)$ is a subsequence of $t_k$ for some $k\ge 1$.
Then $0b01a$ is a prefix of both $(0a1)^\omega$ and $0b(01b)^\omega$.
\label{lemma:longPrefix}
\end{lemma}

\begin{proof}
Let $a = b\cdots$ indicate that $b$ is a prefix of word $a$ and 
consider the statements 
\begin{align*}
A(a,b) : (a10)^\omega = b\cdots \qquad \text{and} \qquad B(a,b) : (b01)^\omega = a\cdots.
\end{align*}
It suffices to show that $A(a,b)$ and $B(a,b)$ are true for any adjacent words $0a1,0b1$ in $t_k$ for 
$k\ge 0$. 
This is because 
\begin{align*}
A(a,b) &\Rightarrow (0a1)^\omega = 0a10(a10)^\omega = 0a10b\cdots = 0b01a\cdots
\intertext{
where the last equality follows from Lemma~\ref{lemma:a10b} and }
B(a,b) &\Rightarrow 0b(01b)^\omega = 0b01(b01)^\omega = 0b01a \cdots
\end{align*}
which are the claims of the Lemma.

We shall use induction. Take $t_2 = (0,001,01,011,01)$ as the base case. We must show that
$A(0,\epsilon), B(0,\epsilon), A(\epsilon, 1), B(\epsilon,1)$ are true. 
However these statements are respectively that $(001)^\omega = \epsilon \cdots, (01)^\omega = 0\cdots, (10)^\omega = 1\cdots, (101)^\omega = \epsilon \cdots$ and are all true.

Otherwise, say $A(a,b), B(a,b)$ are true for any adjacency $0a1,0b1$ in $t_k$.
Let $0a10b1 = 0c1$ so  
\begin{align*}
c = a10b = b01a
\end{align*}
using Lemma~\ref{lemma:a10b} again.
Then the statements $A(a,c),B(a,c),A(c,b),B(c,b)$ are all true as
\begin{align*}
(a10)^\omega &= a10(a10)^\omega = a10b\cdots = c\cdots & \text{by $A(a,b)$ and as $c=a10b$} \\ 
(c01)^\omega &= c\cdots = a\cdots & \text{as $c=a10b$} \\
(c10)^\omega &= c\cdots = b\cdots & \text{as $c=b01a$} \\
(b01)^\omega &= b01(b01)^\omega = b01a\cdots = c\cdots &\text{by $B(a,b)$ and as $c=b01a$.}
\end{align*}
Thus $A(a,b),B(a,b)$ are true for all adjacencies $0a1,0b1$ in $t_{k+1}$.
This completes the proof.
\end{proof}

\begin{lemma}
Suppose $0a1,0b1$ are adjacent in $t_k$ and that $0c1$ lies strictly between them in $t_{k'}$ for some $0< k< k'$. Then $0c1 = 0b01a\cdots$.
\label{lemma:betweenWords}
\end{lemma}
\begin{proof}
The interval of $t_{k'}$ between $0a1,0b1$ is constructed from $0a1,0b1$ in exactly the same way 
as $t_{k' - k}$ was constructed from $0,1$. 
Thus $0c1 = (0a1)^q0b1\cdots$ for some positive integer $q$.
Now recall that $0b01a = 0a10b$ by Lemma~\ref{lemma:a10b}.
Thus $0c1 = (0a1)^{q-1} 0a10b1 \cdots = (0a1)^{q-1} 0b01a1 \cdots = 0b(01a)^q 1\cdots  = 0b01a \cdots$
as claimed.
\end{proof}

Although the existence of a long common prefix for nearby words suggests continuity, to prove anything we must bound the residual after removing the long common prefix. The following Lemma is one way to achieve this.
\begin{lemma}
Suppose $x\ge y\ge 0$, let $0w1$ be the $x$-threshold word 
and let $(01w)^\omega = s u, (10w)^\omega = s' u'$ where $\abs{s} = \abs{s'}$.
Then 
$\abs{S(u,\phi_s(y)) - S(u',\phi_{s'}(y))} \le \frac{x+1}{1-\beta} .$
\label{lemma:prefixS}
\end{lemma}

\begin{proof}
The highest point on the orbits
$(\phi_{((01w)^\omega)_{1:k}}(x) : k \ge 0)$ 
and $(\phi_{((10w)^\omega)_{1:k}}(x) : k \ge 0)$ 
is $x+1$ since $0w1$ is the $x$-threshold word.
The terms $a_k, b_k$ of the discounted sums 
\begin{align*}
\text{$S(u,\phi_s(y)) =: \sum_{k=0}^\infty \beta^k a_k$ and $S(u',\phi_{s'}(y)) =: \sum_{k=0}^\infty \beta^k b_k$}
\end{align*}
are from the orbits 
$(\phi_{((01w)^\omega)_{1:k}}(y) : k \ge 0)$ 
and $(\phi_{((10w)^\omega)_{1:k}}(y) : k \ge 0)$
and $\phi_{u''}(x) \ge \phi_{u''}(y)$ for any word $u''$ as $x\ge y$.
Therefore terms $a_k, b_k$, are also no higher than $\phi_0(x) \le x+1$.
Furthermore, terms $a_k, b_k$ are non-negative, so that $\abs{a_k-b_k} \le x+1$.
Thus
$\abs{S(u,\phi_s(y)) - S(u',\phi_{s'}(y))} \le \sum_{k=0}^\infty \beta^k \abs{a_k - b_k} \le \sum_{k=0}^\infty \beta^k (x+1) = \frac{x+1}{1-\beta} .$
\end{proof}

Although it is clear that $\lambda(\pi,x)$ is continuous, a bound on its slope is helpful.
\begin{lemma}
Suppose $x\ge 0$ and that $0w1$ is a valid word. Then $\abs{\lambda'(0w1,x)} \le \frac{1}{(1-\beta)^2}$.
\label{lemma:lambdaPrime}
\end{lemma}
\begin{proof}
The definition of $\lambda(0w1,x)$ gives
\begin{align*}
\abs{\lambda'(0w1,x)} \le \frac{1-\beta^{\abs{0w1}}}{1-\beta} \sum_{k=0}^\infty \beta^k \abs{\phi_{((01w)^\omega)_{1:k}}'(x) - \phi_{((10w)^\omega)_{1:k}}'(x)}
\le \frac{1}{1-\beta} \sum_{k=0}^\infty \beta^k = \frac{1}{(1-\beta)^2} 
\end{align*}
where the second inequality follows as $0 \le \beta^{\abs{0w1}} < 1$ and $0 \le \phi_u'(x) \le 1$ for any word $u$ since
$0\le \phi_1'(x) \le \phi_0'(x) \le 1$.
\end{proof}

We use one more result about $\phi_0, \phi_1$ of the main paper.
\begin{lemma} Suppose $\phi_0(x)$ and $\phi_1(x)$ are as in the main paper and $x\in\R_+$.
Then $\phi_{01}(x)<\phi_{10}(x)$.
\label{lem:phi0110}
\end{lemma}
\begin{proof}
The definitions of $\phi_0, \phi_1$ give
\begin{align*}
&\phi_{10}(x)-\phi_{01}(x)=\\
%&\quad(b-a) \frac{a b x^2+b x^2+a x^2+2 a b x+3 b x+3 a x+2 x+a b+2 b+2 a+3}
%{(a b x+b x+a x+a b+b+2 a+1) (a b x+b x+a x+a b+2 b+a+1)} 
&\quad(b-a) \frac{(a b +b +a) x^2+(2 a b +3 b +3 a +2) x+a b+2 b+2 a+3}
{((a b +b +a) x+a b+b+2 a+1) ((a b +b +a )x+a b+2 b+a+1)} 
\end{align*}
which is positive as $b>a$ and $x\ge 0$.
\end{proof}

Our proof of continuity will rely on the standard $(\epsilon,\delta)$ definition
in which we will put $\delta = l_k$ where $l_k$ is defined in the following Lemma.

\begin{lemma}
For any $\epsilon > 0$ there is a $k<\infty$ such that $0 < l_k := \inf \{ \abs{X_\pi} : \pi \in t_k \} < \epsilon.$
\label{lemma:lk}
\end{lemma}
\begin{proof}
Say $0a1, 0b1$ are adjacent in $t_k$. 
Then by construction of $t_{k+i}$, 
the gap $(z_{10b}, z_{01a})$ contains $2^i - 1$ intervals corresponding to words of
$t_{k+i} \backslash t_k$.
Each of these intervals is at most $\frac{z_{01a}-z_{10b}}{2^i - 1}$ in length.
Thus $\lim_{k\rightarrow\infty} l_k = 0$.
This demonstrates the existence of a $k<\infty$ such that $l_k < \epsilon$.

To show that $l_k > 0$ for finite $k$, we shall demonstrate that assuming $l_k = 0$ leads to a contradiction. If $l_k = 0$ then there is some word $0w1\in t_k$ such that $z_{10w} = z_{01w} =: x$. Therefore $\phi_{10w}(x) = \phi_{01w}(x)$.
Now in $\R_+$, functions $\phi_{0}(x), \phi_1(x)$ have inverses, so $\phi_{w}^{-1}(x)$ is well-defined. 
Therefore 
\begin{align*}
\phi_{10}(x) =\phi_w^{-1} \circ \phi_{10w}(x) = \phi_w^{-1} \circ \phi_{01w}(x) = \phi_{01}(x)
\end{align*}
which contradicts~Lemma~\ref{lem:phi0110} as $x\ge 0$.
\end{proof}

\subsection{Proof of Continuity}
\begin{proof}
We wish to show that for any $\epsilon > 0$, there exists a $\delta > 0$
such that for any $\abs{x-y} < \delta$ we have $\Delta := \abs{\lambda(x) - \lambda(y)} < \epsilon$. Without loss of generality we assume that $x \ge y$.

Specifically, we shall put $\delta = l_k > 0$ where $l_k$ is as defined in Lemma~\ref{lemma:lk} 
and $k$ is any positive integer such that $\frac{l_k}{(1-\beta)^2} < \frac{\epsilon}{2}$
and such that 
$2 \frac{x+1}{(1-\beta)^2} \beta^{k+1} < \frac{\epsilon}{2}$.
The existence of such a $k$ is guaranteed by Lemma~\ref{lemma:lk}
and because $\beta \in (0,1)$.

Let $0a1,0b1$ be the $x$- and $y$-threshold words.
If these words are the same then
\begin{align*}
\Delta = \abs{\lambda(0a1,x) - \lambda(0a1,y)} \le \abs{y-x} \sup_{z\in [x,y]} \abs{\lambda'(0a1,z)} \le \frac{\abs{y-x}}{(1-\beta)^2} \le \frac{l_k}{(1-\beta)^2} < \frac{\epsilon}{2}
\end{align*}
where the second inequality follows from Lemma~\ref{lemma:lambdaPrime}, 
the third from $\abs{y-x} < \delta = l_k$ 
and the fourth from the definition of $k$.

Otherwise $0a1 \neq 0b1$. In this case, let $(0e1,0b1)$ be the standard pair for word $0b1$,
let $\underline{a} = \phi_{10a}(\underline{a})$ and $\bar b = \phi_{01b}(\bar b)$.
Noting that $y \le \bar b \le \underline{a} \le x$,
our strategy is to write 
\begin{align*}
\Delta &= \abs{\Delta_1 + \Delta_2 + \Delta_3 + \Delta_4 + \Delta_5 + \Delta_6} \\
\Delta_1 &:= \lambda(0b1,y) - \lambda(0b1,\bar b) \\
\Delta_2 &:= \lambda(0b1,\bar b) - \lambda(0e1(0b1)^\omega, \bar b) \\
\Delta_3 &:= \lambda(0e1(0b1)^\omega, \bar b) - \lambda((0a1)^\omega, \bar b) \\
\Delta_4 &:= \lambda((0a1)^\omega, \bar b) - \lambda((0a1)^\omega, \underline a) \\
\Delta_5 &:= \lambda((0a1)^\omega, \underline a) - \lambda(0a1, \underline a) \\
\Delta_6 &:= \lambda(0a1, \underline a) - \lambda(0a1, x) .
\end{align*}

Lemma~\ref{lemma:lambdaPrime} and the choice of $\delta$ give
\begin{align}
\abs{\Delta_1} + \abs{\Delta_4} + \abs{\Delta_6} \le \frac{\bar b - y + \underline a - \bar b + x - \underline a}{(1-\beta)^2} < \frac{l_k}{(1-\beta)^2} \le \frac{\epsilon}{2}
\label{delta:146}
\end{align}
while Lemmas~\ref{lemma:leftSide} and~\ref{lemma:rightSide} give
\begin{align}
\Delta_2 = \Delta_5 = 0.
\label{delta:25}
\end{align}

It remains to consider $\Delta_3$.
It follows from the definition of $l_k$, that for some adjacent words $0c1, 0d1$
in $t_k$:
either $0a1 = 0c1$ or $0a1$ is a word strictly between $0c1$ and $0d1$ in the sense of Lemma~\ref{lemma:betweenWords};
and that $0e1(0b1)^\omega$ is a word strictly between $0c1$ and $0d1$.
Thus by Lemma~\ref{lemma:betweenWords} we have $(0a1)^\omega = 0 p u$
and $0e1(0b1)^\omega = 0 p v$ where $p := d01c$ and $u,v$ are the appropriate suffixes. Therefore the definition of $\lambda(w,x)$ gives
\begin{align}
\abs{\Delta_3} 
&= \abs{\lambda((0a1)^\omega,\bar b) - \lambda(0d1(0b1)^\omega, \bar b)}  \nonumber \\
&= \frac{1}{1-\beta} \begin{vmatrix} S(01p,\bar b) + \beta^{\abs{01p}} S(u, \phi_{01p}(\bar b))  - S(10p,\bar b) - \beta^{\abs{10p}} S(u, \phi_{10p}(\bar b)) \nonumber \\
- S(01p,\bar b) - \beta^{\abs{01p}} S(v, \phi_{01p}(\bar b))
+ S(10p,\bar b) + \beta^{\abs{01p}} S(v, \phi_{10p}(\bar b)) \end{vmatrix}
\nonumber  \\
&= \frac{\beta^{\abs{01p}}}{1-\beta} \abs{  S(u, \phi_{01p}(\bar b))
 - S(u, \phi_{10p}(\bar b))
- S(v, \phi_{01p}(\bar b))
+ S(v, \phi_{10p}(\bar b))
} \nonumber \\
&\le \frac{\beta^{\abs{01p}}}{1-\beta} \left( \abs{S(u, \phi_{01p}(\bar b))
 - S(u, \phi_{10p}(\bar b))} + \abs{S(v, \phi_{01p}(\bar b)) - S(v, \phi_{10p}(\bar b))} \right) \nonumber \\
&\le  \frac{\beta^{\abs{01p}}}{1-\beta} \left( \frac{\underline{a}+1}{1-\beta} + \frac{\bar b + 1}{1-\beta} \right)  \nonumber \\
&\le \frac{\beta^{k+1}}{(1-\beta)^2} 2(x+1) \nonumber \\
&< \frac{\epsilon}{2} 
\label{delta:3}
\end{align}
where the last four inequalities follow from the triangle inequality, 
from Lemma~\ref{lemma:prefixS}, 
from equation~\ref{remark:uv} coupled with the fact that $\underline{a} \le \bar b \le x$ and
finally from the definition of $k$.

Finally, coupling~\ref{delta:146},~\ref{delta:25} and~\ref{delta:3} and using the triangle inequality gives
\begin{align*}
\Delta &< \frac{\epsilon}{2} + 0 + \frac{\epsilon}{2} = \epsilon.
\end{align*}
This completes the proof.
\end{proof}
%\vfill\pagebreak

\section{Properties of the Linear-System Orbits $M(w)$}
\newcommand\myV{\vspace{0.2cm}}
\newcommand\myVV{\vspace{0.4cm}}
Recall the definitions about words from the main paper,
particularly that $\tilde w$ is the reverse of $w$.
Also, recall the definitions of matrices $F,G,K,M(w)$.
The first of the following propositions is used to prove
the second. The second appears in the main paper.

\begin{proposition}
Suppose $w,w'$ are any words.
Then
%\begin{enumerate}
%\item $\det(M(w))=1$,
%\item $M(\tilde w)=KM(w)^{-1}K$,
%\item $M(w)=\begin{pmatrix}e&f\\\frac{eh-1}{f}&h\end{pmatrix}$ for some $e,f,h\in\R$,
%\item $M(w)-M(\tilde w) = \lambda K$ for some $\lambda\in\R$,
%% \item $\frac{[M(w')GFM(w)]_{22}}{[M(w')GFM(w)]_{21}} \ge \frac{[M(w')FGM(w)]_{22}}{[M(w')FGM(w)]_{21}}$,
%\item $\displaystyle \frac{[M(w01w')]_{22}}{[M(w01w')]_{21}} \ge \frac{[M(w10w')]_{22}}{[M(w10w')]_{21}}$,
%\item $[M(w)]_{22}\ge [M(w)]_{21}$.
%\end{enumerate}
\begin{enumerate}
\item $\det(M(w))=1,$ 
\item $M(\tilde w)=KM(w)^{-1}K$,
\item $M(w)=\begin{pmatrix}e&f\\\frac{eh-1}{f}&h\end{pmatrix}$ for some $e,f,h\in\R$, 
\item $M(w)-M(\tilde w) = \lambda K$ for some $\lambda\in\R$,
\item $\displaystyle \frac{[M(w01w')]_{22}}{[M(w01w')]_{21}} \ge \frac{[M(w10w')]_{22}}{[M(w10w')]_{21}}$,
\item $[M(w)]_{22}\ge [M(w)]_{21}$.
\end{enumerate}
%\begin{tabular}{ll}
%1. $\det(M(w))=1,$ & 4. $M(w)-M(\tilde w) = \lambda K$ for some $\lambda\in\R$,\\
%2. $M(\tilde w)=KM(w)^{-1}K$, & 5. $\displaystyle \frac{[M(w01w')]_{22}}{[M(w01w')]_{21}} \ge \frac{[M(w10w')]_{22}}{[M(w10w')]_{21}}$,\\
%3. $M(w)=\begin{pmatrix}e&f\\\frac{eh-1}{f}&h\end{pmatrix}$ for some $e,f,h\in\R$, & 6. $[M(w)]_{22}\ge [M(w)]_{21}$.
%\end{tabular}
\label{prop:M}
\end{proposition}
\begin{proof} $\det(M(w))=\prod_{i=1}^{\abs{w}} \det(M(w_i)) = 1$ as $\det(F)=\det(G)=1$ gives {\it Claim 1.}
\myV\\ 
{\it Claim~2.}~The definitions of $F,G,K$ give $KF=F^{-1}K, KG=G^{-1}K$. Thus
$KM(w) = M(w_{\abs{w}})^{-1} \cdots M(w_1)^{-1}K = M(\tilde w)^{-1}K$.
The result follows as $K^2=I$. \\ 
{\it Claim~3.}~Put $M(w) =: \begin{pmatrix} e&f\\g&h\end{pmatrix}$
and solve $\det(M(w))=1=eg-hf$ for $g$.
\\
{\it Claim~4.}~Substituting Claim~2 and Claim~3 in Claim~4 gives $M(w)-KM(w)^{-1}K = (h-e-g) K.$
\\
{\it Claim~5.}~Put $M:=M(w), N:=M(w')$.
We calculate
\begin{align*}
&[NGFM]_{22} [NFGM]_{21} - [NGFM]_{21} [NFGM]_{22} \\
&\quad = (b-a) (M_{11} M_{22} - M_{12} M_{21}) ((ab+b+a)N_{22}^2+(b+a+2) N_{21} N_{22} + N_{21}^2)
\ge 0
\end{align*}
as $b>a\ge 0$, $\det(M) = 1$ and $N\ge 0$. 
The result follows as $NFGM\ge 0$ and $NGFM\ge 0$.
\myV\\
{\it Claim~6.}~
If $w=\epsilon$ then $[M(w)]_{22}-[M(w)]_{21} = 1 \ge 0$.
Otherwise we use induction 
on $\abs{w}$ to show that $M(w) v\ge 0$
where $v:=(-1,1)^T$. 
In the base case $w\in\{0,1\}$ so 
\begin{align*}
M(w)v = \begin{pmatrix}1&1\\c&1+c\end{pmatrix} \begin{pmatrix}-1\\1\end{pmatrix} = \begin{pmatrix} 0\\ 1\end{pmatrix} \ge 0 \qquad\text{for some $c\in\{a,b\}$}.
\end{align*}
For the inductive step, assume $w = \{0u,1u\}$ for some word $u$ satisfying $M(u)v\ge 0$.
Then 
\begin{align*}
M(w)v = \begin{pmatrix}1&1\\c&1+c\end{pmatrix} M(u)v \ge 0 \qquad \text{for some $c\in\{a,b\}$.}
\end{align*}
As $[M(w)v]_2 = [M(w)]_{22}-[M(w)]_{21}$, this completes the proof.
\end{proof}
%\vfill\pagebreak

\begin{proposition}
Suppose $w$ is a word, $p$ is a palindrome and $n\ge \Z_+$.
Then
\begin{enumerate}
\item $M(p)=\begin{pmatrix}\frac{fh+1}{h+f}&f\\\frac{h^2-1}{h+f}&h\end{pmatrix}$ for some $f,h\in\R$,
\item $\text{tr}(M(10p))=\text{tr}(M(01p))$,
\item If $u\in\{p(10p)^n,(10p)^n10\}$ then 
$M(u)-M(\tilde u) = \lambda K$ for some $\lambda\in\R_-$,
\item If $w$ is a prefix of $p$ then $[M(p(10p)^n10w)]_{22}\le [M(p(01p)^n01w)]_{22}$,
\item $[M((10p)^n10w)]_{21}\ge [M((01p)^n01w)]_{21}$,
\item $[M((10p)^n1)]_{21}\ge[M((01p)^n0)]_{21}$.
\end{enumerate}
\label{prop:pal}
\end{proposition}
\begin{proof} In this proof, we refer to Claim~$k$ of Proposition~\ref{prop:M} as P$k$.\myV\\
{\it Claim 1.} P2 gives $M(p)=KM(p)^{-1}K$ as $p=\tilde p$.
But in the notation of P3, 
$[M(p)]_{11} = [K M(p)^{-1} K]_{11}$ says $e=h-(eh-1)/f$.
Solve this for $e$ and substitute in P3. 
\myV\\
{\it Claim 2.} Noting that $GF-FG=(b-a)K$,
 the notation of Claim~1 gives
\begin{align*}
\text{tr}(M(01p))-\text{tr}(M(10p)) = \text{tr}(M(p)(GF-GF))
= (b-a) \text{tr}\left( 
\begin{pmatrix}\frac{fh+1}{h+f}&f\\\frac{h^2-1}{h+f}&h\end{pmatrix} K \right) 
= 0.
\end{align*}
{\it Claim 3.} 
Note we can move from $u$ to $\tilde u$ just by swapping some $10$ for $01$. So, repeated application of P5 gives the inequality
$
\frac{[M(u)]_{22}}{[M(u)]_{21}}\le \frac{[M(\tilde u)]_{22}}{[M(\tilde u)]_{21}} .
$
But the denominators of this inequality are equal (and non-negative)
as P4 gives $[M(u)]_{21} - [M(\tilde u)]_{21} = \lambda' K_{21} = 0$ for some $\lambda'\in\R$.
Thus this inequality reduces to 
$[M(u)]_{22}\le [M(\tilde u)]_{22}$.
Yet P4 also gives $[M(u)-M(\tilde u)]_{22} = \lambda K_{22}$ 
which combined with the previous sentence says that $\lambda K_{22} \le 0$. As $K_{22}=1$, this gives $\lambda \in \R_-$.
\myV\\
{\it Claim 4.} Let $s$ be the corresponding suffix so $p=ws$
and $$M(p(10p)^n10w)-M(p(01p)^n01w)=M(s)^{-1}(M(p(10p)^{n+1})-M(p(01p)^{n+1}))=:A.$$
But Claim~3 with $u=p(10p)^{n+1}$ gives 
%\begin{align*}
%[M(s)^{-1}(M(p(10p)^{n+1})-M(p(01p)^{n+1}))]_{22} &= \lambda [M(s)^{-1}K]_{22} 
%&\text{for some $\lambda\le 0$} \\
%&=[KM(\tilde s)]_{22}&\text{by P2} \\
%&= \lambda ([M(\tilde s)]_{22}-[M(\tilde s)]_{21}) \\
%&\le 0 &\text{by P6} .
%\end{align*}
\begin{align*}
[A]_{22} &= \underbrace{\lambda [M(s)^{-1}K]_{22}}_{\text{for some $\lambda\le 0$}} = \underbrace{[KM(\tilde s)]_{22}}_{\text{by P2}}
= \lambda ([M(\tilde s)]_{22}-[M(\tilde s)]_{21}) \le {\underbrace 0}_{\text{by P6.}}
\end{align*}
\\
{\it Claim 5.}
As $M(w)\ge 0$, Claim~3 with $u=(10p)^n10$ gives
\begin{align*}
[M(w)(M((10p)^n10)-M((01p)^n01))]_{21} = \lambda [M(w)K]_{21} = \lambda [-M(w)]_{21} \ge 0.
\end{align*}
\\
{\it Claim 6.} Let $E:=\begin{pmatrix}0&0\\1&1\end{pmatrix}.$
Then $G-F=(b-a)E\ge 0$, so that
\begin{align*}
[GM((10p)^n)-FM((01p)^n)]_{21}
&= [(b-a)EM((10p)^n)+FM((10p)^n)-FM((01p)^n)]_{21}\\
&\ge [M((10p)^n0)-M((01p)^n0)]_{21} \ge 0 
\end{align*} 
by Claim~5. This completes the proof.
\end{proof}

\section{Majorisation}
In the main paper, we used one result about majorisation which was similar-but-not-identical to 
any results in Marshall, Olson and Arnold (2011). 
Let us prove that result.

\begin{proposition}
Suppose $x,y \in \R_+^m$ and
$f : \R \rightarrow \R$ is a symmetric function that is convex and
decreasing on $\R_+$.
Then 
$\text{$x\prec^w y$ and $\beta\in [0,1]$} \quad\Rightarrow \quad\sum_{i=1}^m \beta^{i} f(x_{(i)}) \ge \sum_{i=1}^m \beta^{i} f(y_{(i)})$.
\label{prop:symconvex}
\end{proposition}
\begin{proof}
As the claim relates to $x_{(i)}$ and $y_{(i)}$ we assume that
$x_i$ and $y_i$ are in ascending order.

Marshall {\it et al} (3H2B, page 133) says that if $g : \mathcal{A}\rightarrow \R$ is a non-decreasing and convex function on 
$\mathcal{A}\subseteq \R$
and $(u_1, \dots, u_m)$ is a non-increasing and non-negative sequence, then 
for all non-increasing sequences $(p_1, \dots, p_m)$
the function $\phi(a) := \sum_{i=1}^m u_i g(p_i)$ is Schur-convex.

Indeed the function $f$ is increasing and convex for $p\in \R_-$ 
(such as $p=-x$ and $p=-y$) 
and $(\beta, \dots, \beta^m)$ is a non-increasing
and non-negative sequence for $\beta \in [0,1]$. Thus for all non-increasing sequences 
$(p_1, \dots, p_m)$ on $\R_-^m$ the function $\psi(p) := \sum_{i=1}^m \beta^i f(p_i)$ is Schur-convex. 

Recall ({\it ibid}, page 12) that $a\in\R^m$ is said to be 
{\it weakly submajorised} by $b\in\R^m$, written $a\prec_w b$ if 
\begin{align*}
\sum_{i=1}^k a_{[i]} \le \sum_{i=1}^k b_{[i]}, \quad k=1, \dots, m
\qquad\text{where $a_{[i]}$ denotes $a$ in descending order}
\end{align*}
and that $x\prec_w y \Leftrightarrow -a\prec^w -b$ ({\it ibid}, page 13).

However ({\it ibid}, 3A8, page 87) if $\phi(p)$ is a real function on $\mathcal{A}\subset \R^m$ 
which is non-decreasing in each argument $p_i$ and Schur-convex on $\mathcal{A}$ and 
$p\prec_w q$ on $\mathcal{A}$ then $\phi(p) \le \phi(q)$.

Indeed, the function $\psi(p)= \sum_{i=1}^m \beta^i f(p_i)$ is a real function on $\R_-^m$ which is non-decreasing in each argument and Schur-convex
on $\R_-^m$ for all non-increasing sequences $(p_1, \dots, p_m)$.
Furthermore, $-y \prec_w -x$ as $x \prec^w y$.
Therefore $\psi(y) = \psi(-y) \le \psi(-x) = \psi(x)$ as claimed.
\end{proof}
%\vfill\pagebreak

\section{Clarification of Theorem~1 for $0\le x\le y_1$ or $y_0\le x < \infty$}
Recall the following definitions and assumption from the main paper
\begin{align*}
F&:=\begin{pmatrix}1&1\\a&1+a\end{pmatrix},&
G&:=\begin{pmatrix}1&1\\b&1+b\end{pmatrix},&
E&:=\begin{pmatrix}0&0\\1&1\end{pmatrix},&
v(x)&:=\begin{pmatrix}z\\1\end{pmatrix},
&b>a\ge 0.
\end{align*}
If $0\le x\le y_1$ or $y_0\le x < \infty$ then the relevant linear systems, (9) in the main paper, are 
\begin{align*}
\left.
\begin{aligned}
(M(1^{k+1}) - M(01^{k}))v(x)&=(G-F)G^{k} v(x)=(b-a) E G^{k} v(x) \ge 0\\
(M(10^{k}) - M(0^{k+1}))v(x)&=(G-F)F^{k} v(x)=(b-a) E F^{k} v(x) \ge 0
\end{aligned}\right\} \quad \text{for $k\in Z_+$}
\end{align*}
where both inequalities follow 
as $E, F, G$ are all $\ge 0$,
as $b>a$
and as 
$x\ge\min\{y_0,y_1\}\ge 0.$
Therefore all cumulative sums of the above expressions are non-negative
so the derivative of the numerator of the Whittle index is non-negative by the same weak-supermajorisation argument as in the main paper. 

Meanwhile, the denominator of the index in these cases is
\begin{align*}
\sum_{k=0}^\infty \beta^k ((1^\omega)_{k+1}-(01^\omega)_{k+1})
= \beta = \sum_{k=0}^\infty \beta^k ((10^\omega)_{k+1}-(0^\omega)_{k+1})
\end{align*}
which is non-negative.
Therefore the rest of the proof of Theorem~1 follows as in the main paper. 

In fact we could say that the majorisation point, which is $\phi_w(0)$ for words $0w1$ in the main paper, is $-1$ in both cases.
Indeed, 
Claim~6 of Proposition~4 of the main paper says that
$F v(-1) = G v(-1) = v(0)$. Also,
$E v(-1) = (0,0)^T.$
Thus for all $k\in \Z_+$, 
$E G^{k} v(-1) \ge E F^{k} v(-1) \ge 0$
whereas $E v(-1-\epsilon) < 0$ for any $\epsilon >0.$
\end{document}